%% file: paper.tex
\documentclass[runningheads]{llncs}
\usepackage[T1]{fontenc}
\usepackage{pifont}
\usepackage{algorithm}
\usepackage{algorithmic}
\usepackage{hyperref}
\usepackage{subfigure}
\usepackage{comment}
\usepackage{amsmath,amssymb}
\usepackage{color,xcolor,colortbl}
\usepackage{epsfig}
\usepackage{multirow}
\usepackage{array}
\usepackage{subfigure}
\usepackage{setspace}
\usepackage{wrapfig}
\usepackage{arydshln}
\usepackage{autobreak}
\usepackage{ragged2e}
\usepackage{scalerel}
\usepackage{enumitem}
\usepackage{float}
\usepackage{adjustbox}
\usepackage{booktabs}
\usepackage{marvosym}
\usepackage{graphicx}
\usepackage[square, comma, numbers]{natbib}
\usepackage{subfigure}
\makeatletter
\def\@biblabel#1{#1.}
\makeatother

\begin{document}
\title{Learning Sparse Neural Networks with Identity Layers}
%
%
\author{Mingjian Ni\inst{1} \and
Guangyao Chen\inst{1} \and
Xiawu Zheng\inst{2} \and
Peixi Peng\inst{1} \and
Li Yuan\inst{1} \and
Yonghong Tian\inst{1} \textsuperscript{(\Letter)}}
\authorrunning{M. Ni et al.}
%
\institute{Peking University, Beijing 100871, China \\ \email{\{sccdnmj, gy.chen, pxpeng, yuanli-ece, yhtian\}@pku.edu.cn} \and
Peng Cheng Laboratory, Shenzhen 518055, China \\ \email{zhengxw01@pcl.ac.cn}}

\maketitle              
\begin{abstract}
  The sparsity of Deep Neural Networks is well investigated to maximize the performance and reduce the size of overparameterized networks as possible. Existing methods focus on pruning parameters in the training process by using thresholds and metrics. Meanwhile, feature similarity between different layers has not been discussed sufficiently before, which could be rigorously proved to be highly correlated to the network sparsity in this paper. 
  Inspired by interlayer feature similarity in overparameterized models, we investigate the intrinsic link between network sparsity and interlayer feature similarity. Specifically, we prove that reducing interlayer feature similarity based on Centered Kernel Alignment (CKA) improves the sparsity of the network by using information bottleneck theory.
  Applying such theory, we propose a plug-and-play \textbf{CKA}-based \textbf{S}parsity \textbf{R}egularization for sparse network training, dubbed CKA-SR, which utilizes CKA to reduce feature similarity between layers and increase network sparsity. In other words, layers of our sparse network tend to have their own identity compared to each other. 
  Experimentally, we plug the proposed CKA-SR into the training process of sparse network training methods and find that CKA-SR consistently improves the performance of several State-Of-The-Art sparse training methods, especially at extremely high sparsity. 
  Code is included in the supplementary materials.

\keywords{Network sparsity \and Inter-layer feature similarity \and Network compression.}
\end{abstract}

\input{Tex/01_Introduction}
\input{Tex/03_Preliminaries}
\input{Tex/04_Methodology}
\input{Tex/05_Experiments}
\input{Tex/06_Conclusion}

%
%
%
%




\renewcommand\bibname{References}
\makeatletter
\renewenvironment{thebibliography}[1]
  {\section*{\bibname}%
   \@mkboth{\MakeUppercase\bibname}{\MakeUppercase\bibname}%
   \list{\@biblabel{\@arabic\c@enumiv}}%
        {\settowidth\labelwidth{\@biblabel{#1}}%
         \leftmargin\labelwidth
         \advance\leftmargin\labelsep
         \@openbib@code
         \usecounter{enumiv}%
         \let\p@enumiv\@empty
         \renewcommand\theenumiv{\@arabic\c@enumiv}}%
   \sloppy
   \clubpenalty4000
   \@clubpenalty \clubpenalty
   \widowpenalty4000%
   \sfcode`\.\@m}
  {\def\@noitemerr
    {\@latex@warning{Empty `thebibliography' environment}}%
   \endlist}
\makeatother
\bibliography{references}
\end{document}


\title{Appendix}
%
%
\author{Anonymous Authors}
\institute{}

\maketitle              

\input{Tex/02_Related_Works}

\section{Additional analysis}
For convenience, we abbreviate $\textbf{CKA}_{Linear}$ as $\textbf{CKA}_L$ in the appendix.
\subsection{Proof of two lemmas}

\begin{lemma}
\label{Thm: lemma1}
Minimizing the distance between $X^TY$ and zero matrix is equivalent to minimizing the mutual information $I(X; Y)$ between representation $X$ and $Y$. 
\end{lemma}
\begin{lemma}
\label{Thm: lemma2}
Minimizing $\mathbf{CKA}_{L}(X_i, X_j)$ is equivalent to minimizing $I(X_i; X_j)$.
\end{lemma}
We prove the two Lemmas \cite{zheng2021information} as follows.

\begin{proof}
According to the definition of mutual information, $I(X; Y) = H(X) + H(Y) - H(X, Y)$. And according to the definition of multivariate Gaussian distribution, $H(X) = \frac{1}{2}ln((2\pi e)^D|\Sigma_X|)$. We assume that $X$ and $Y$ obey the multivariate Gaussian distribution, then we have $(X, Y) \sim N(0, \Sigma_{(X, Y)})$, in which
\begin{equation}
\label{eq1}
\Sigma_{(X, Y)} = 
	\begin{pmatrix}
	 \Sigma_X & \Sigma_{XY} \\
     \Sigma_{YX} & \Sigma_Y
	 \end{pmatrix}
\end{equation}
The assumption that $X$ and $Y$ obey the multivariate Gaussian distribution with zero mean is reasonable, for there exist batch normalization layers in neural networks.
Therefore, the mutual information between random variables $X$ and $Y$ following a Gaussian distribution is formulated as,
\begin{align}
\label{eq2}
  I(X; Y) = ln|\Sigma_X| + ln|\Sigma_Y| - ln|\Sigma_{(X, Y)}|
\end{align}
According to Everitt inequality, $|\Sigma_{(X, Y)}| \leq |\Sigma_X||\Sigma_Y|$, this inequality holds if and only if $\Sigma_{YX} = \Sigma_{XY^T} = X^TY$ is a zero matrix. Besides, the definition of linear HSIC is as follows, 
\begin{align}
\label{eq3}
    \mathrm{nHSIC}_{linear}(X, Y) = \frac{||Y^TX||_F^2}{||X^TX||_F||Y^TY||_F}
\end{align}
And according to the definition of the Frobenius norm, we have,
\begin{align}
\label{eq4}
    ||Y^TX||_F^2 = ||X^TY||_F^2
\end{align}
Combining Eq.~\eqref{eq2}, Eq.~\eqref{eq3}, Eq.~\eqref{eq4}, and Everitt inequality, we have, $I(X; Y) \geq 0$, and this inequality holds if and only if $X^TY$ is a zero matrix. Further, we have: When minimizing Eq.~\eqref{eq3}, we are also minimizing the distance between $X^TY = X^Tf(X)$ and the zero matrix, \textit{i.e.} \textbf{Lemma1}. (As $X$ and $Y$ are feature representations produced by different layers of the same network, we can represent $Y$ as $Y=f(X)$, in which $f$ is the sub-network between $X$ and $Y$. As a sub-network of the neural network, $f$ is a nonlinear function.) In other words, while minimizing $nHSIC_{linear}$, we also minimize the mutual information between the two Gaussian distributions, \textit{i.e.} \textbf{Lemma2}.
\end{proof}



\subsection{Proof of Theorem 1}

\begin{theorem}
\label{Thm: theorem1}
The $\epsilon$-sparsity and the sparsity of neural networks are approximately equivalent.
\end{theorem}

\begin{proof}

For convenience, we consider a neural network $N$ consisting of a full-connected layer. We mark the sparsity of $N$ as $s$. For a given input $X$, the network $N$ outputs $Y=WX+b$, in which $W$ and $b$ are the weight matrix and bias, respectively.

We mark the minimum non-zero absolute value of weight matrix $W$ as $min(W)$. We take an $\epsilon$ which meet the conditions of $\epsilon \rightarrow 0$ and $\epsilon << min(W)$. Obviously, as long as $\epsilon$ approaches 0, both conditions above are satisfied. Then we replace 0 values of the weight matrix $W$ with values smaller than $\epsilon$, and name this weight matrix as $W'$. We mark the neural network with weight matrix $W'$ as $N'$. Because we have $\epsilon << min(W)$, the $\epsilon$-sparsity of network $N'$ is $s$. The two networks are identical except that $N$ has the sparsity of $s$ while $N'$ has the $\epsilon$-sparsity of $s$.

For a given input $X$, we consider the outputs of the two networks $N$ and $N'$. The output of $N$ is $Y=WX+b$, and the output of $N'$ is $Y'=W'X+b$. We take the difference of $Y$ and $Y'$ as follows:
\begin{align}
\label{Eqn:difference}
    ||Y-Y'||=||(WX+b)-(W'X+b)||=||(W-W')X||
\end{align}
According to the definition of $W'$, each element of the $W-W'$ matrix approaches 0. So we have $||W-W'|| \rightarrow 0$, then we have $||Y-Y'||=||(W-W')X|| \rightarrow 0$, which means the outputs of $N$ and $N'$ are approximately the same. Because the two networks are identical except for their sparsity, we come to the conclusion that $\epsilon$-sparsity and sparsity of neural networks are approximately equivalent in practice.
\end{proof}

\subsection{Proof of Theorem 2}
\begin{theorem}
\label{Thm: theorem2}
$\mathcal{L_{\mathcal{C}}}$ is continuous and optimizable.
\end{theorem}

\begin{proof}
First we prove that $\mathcal{L_{\mathcal{C}}}$ is continuous. Because $\mathcal{L_{\mathcal{C}}}$ is the simple summation of $\mathbf{CKA}_{L}(X_i,X_j)$, we only need to prove $\mathbf{CKA}_{L}(X_i,X_j)$ is continuous for arbitrary element of matrix $X_i$ and $X_j$. Selecting an arbitrary element $X_{ik'l'}$ of matrix $X_i$, we rewrite $\mathbf{CKA}_{L}(X_i,X_j)$ as
\begin{align}
\label{Eqn:continuous1}
\mathbf{CKA}_{L}(X_i,X_j)
=\frac{\sum_{k=1}^n \sum_{l=1}^n (X_i^TX_j)_{kl}^2}{K\sqrt{\sum_{k=1}^n \sum_{l=1}^n (X_i^TX_i)_{kl}^2}},
\end{align}
in which $K=\sqrt{\sum_{k=1}^n \sum_{l=1}^n (X_j^TX_j)_{kl}^2}$ is a constant. Further, we have
\begin{align}
\label{Eqn:continuous2}
(X_i^TX_j)_{kl} = \sum_{t=1}^n \sum_{u=1}^n X_{itk}X_{jul} \\
(X_i^TX_i)_{kl} = \sum_{t=1}^n \sum_{u=1}^n X_{itk}X_{iul}.
\end{align}
Substituting Eq.~\eqref{Eqn:continuous2} into Eq.~\eqref{Eqn:continuous1}, we have
\begin{align}
\label{Eqn:continuous3}
\mathbf{CKA}_{L}(X_i,X_j)
=\frac{\sum_{k=1}^n \sum_{l=1}^n (\sum_{t=1}^n \sum_{u=1}^n X_{itk}X_{jul})^2}{K\sqrt{\sum_{k=1}^n \sum_{l=1}^n (\sum_{t=1}^n \sum_{u=1}^n X_{itk}X_{iul})^2}}.
\end{align}
Obviously, $\sum_{k=1}^n \sum_{l=1}^n (\sum_{t=1}^n \sum_{u=1}^n X_{itk}X_{jul})^2$ and $\sum_{k=1}^n \sum_{l=1}^n (\sum_{t=1}^n \sum_{u=1}^n X_{itk}X_{iul})^2$ are polynomials of $X_{ik'l'}$. We write them as $p_1(X_{ik'l'})$ and $p_2(X_{ik'l'})$, respectively. Then we rewrite $\mathbf{CKA}_{L}(X_i,X_j)$ as
\begin{align}
\label{Eqn:continuous4}
\mathbf{CKA}_{L}(X_i,X_j)=\frac{p_1(X_{ik'l'})}{K\sqrt{p_2(X_{ik'l'})}}.
\end{align}
From Eq.~\eqref{Eqn:continuous4}, it's obvious that $\mathbf{CKA}_{L}(X_i,X_j)$ is derivable.(It's derivative is $\frac{1}{K}p_1'(X_{ik'l'})p_2(X_{ik'l'})^{-\frac{1}{2}}+\frac{-1}{2K}p_1(X_{ik'l'})p_2(X_{ik'l'})^{-\frac{3}{2}}p_2'(X_{ik'l'})$) Because derivable functions must be continuous, we prove that $\mathbf{CKA}_{L}(X_i,X_j)$ is continuous. Therefore, $\mathcal{L_{\mathcal{C}}}$ is continuous.
Further, because $\mathcal{L_{\mathcal{C}}}$ has derivative at domain of function, we could carry out local optimization (by gradient descent) of $\mathcal{L_{\mathcal{C}}}$ to approach extremum values. That is to say, $\mathcal{L_{\mathcal{C}}}$ is optimizable in machine learning.


\end{proof}

\subsection{Proof of Theorem 3}
\begin{theorem}
\label{Thm: theorem3}
Minimizing $\mathcal{L_{\mathcal{C}}}$ minimizes the mutual information $R = I(X;\hat{X})$ between intermediate representation $\hat{X}$ and input representation $X$.
\end{theorem}

\begin{proof} According to Theorem \ref{Thm: theorem2}, CKA-SR is optimizable. In consideration of clear expression, we take a one-stage network or a single stage of multi-stage networks for example.


\begin{equation}
\begin{aligned}
\label{Eqn:proof1}
    & \min \mathcal{L_{\mathcal{C}}} \Leftrightarrow \min \beta \cdot \sum_{s=1}^1 \sum_{i=0}^{N_s} \sum_{j=0}^{N_s} \mathbf{CKA}_{L}(X_i, X_j)\\
    &\Leftrightarrow \min 2\times\sum_{i=0}^{N_s} \sum_{j=i+1}^{N_s} \mathbf{CKA}_{L}(X_i, X_j) + \sum_{i=0}^{N_s} \mathbf{CKA}_{L}(X_i, X_i)\\
    &\Leftrightarrow \min 2\times\sum_{i=0}^{N_s} \sum_{j=i+1}^{N_s} \mathbf{CKA}_{L}(X_i, X_j) + ({N_s}+1) \\
    &\Leftrightarrow \min \sum_{i=0}^{N_s} \sum_{j=i+1}^{N_s} \mathbf{CKA}_{L}(X_i, X_j)
\end{aligned}
\vspace{-3pt}
\end{equation}

We define several successive layers (from the $i^{th}$ to the $j^{th}$ layers) of the original network as a layer-sub-network. Then the $R = I(X;\hat{X})$ of the layer-sub-network can be expressed as $R_{ij} = I(X_i; X_j)$, in which $j>i$, $X_i$ is the input of the layer-sub-network, and $X_j$ is the map of $X_i$ through the layer-sub-network. According to Lemma \ref{Thm: lemma2}, minimizing $\mathbf{CKA}_{L}(X_i, X_j)$ is equivalent to minimizing $I(X_i; X_j)$, \emph{i.e.}, $R_{ij}$. Substituting $R_{ij}$ into the minimization objective, we have:


\begin{equation}
\begin{aligned}
\label{Eqn:proof2}
    \min \sum_{i=0}^{N_s} \sum_{j=i+1}^{N_s} \mathbf{CKA}_{L}(X_i, X_j) \Leftrightarrow \min \sum_{i=0}^{N_s} \sum_{j=i+1}^{N_s} R_{ij}
\end{aligned}
\vspace{-2pt}
\end{equation}

That is, minimizing the CKA-SR is equivalent to minimizing the sum of mutual information $R$ of all layer-sub-networks of the network, thus minimizing the mutual information between intermediate representation and input representation.
\end{proof}

\subsection{Potential negative societal impacts and limitations}
To the best of our knowledge, there exists almost no potential negative societal impact in our method, which is a plug-and-play regularization in sparse training and pruning methods. However, considering the expenses to select the hyperparameter and calculate loss function, it's necessary to think about reducing the energy consumption and protecting the environment while conducting experiments.

\section{Additional experimental results}
\subsection{Computation settings}
We conduct our experiments on a server with 8 Tesla V100 GPUs and 40 CPUs. Specifically, we use 1 Tesla V100 GPU for CIFAR experiments and we use 8 Tesla V100 GPUs for ImageNet experiments.

\begin{figure}[!tb]
	\centering
	\subfigure[Baseline]{
		\begin{minipage}[t]{0.24\linewidth}
			\centering
			\includegraphics[width=\linewidth]{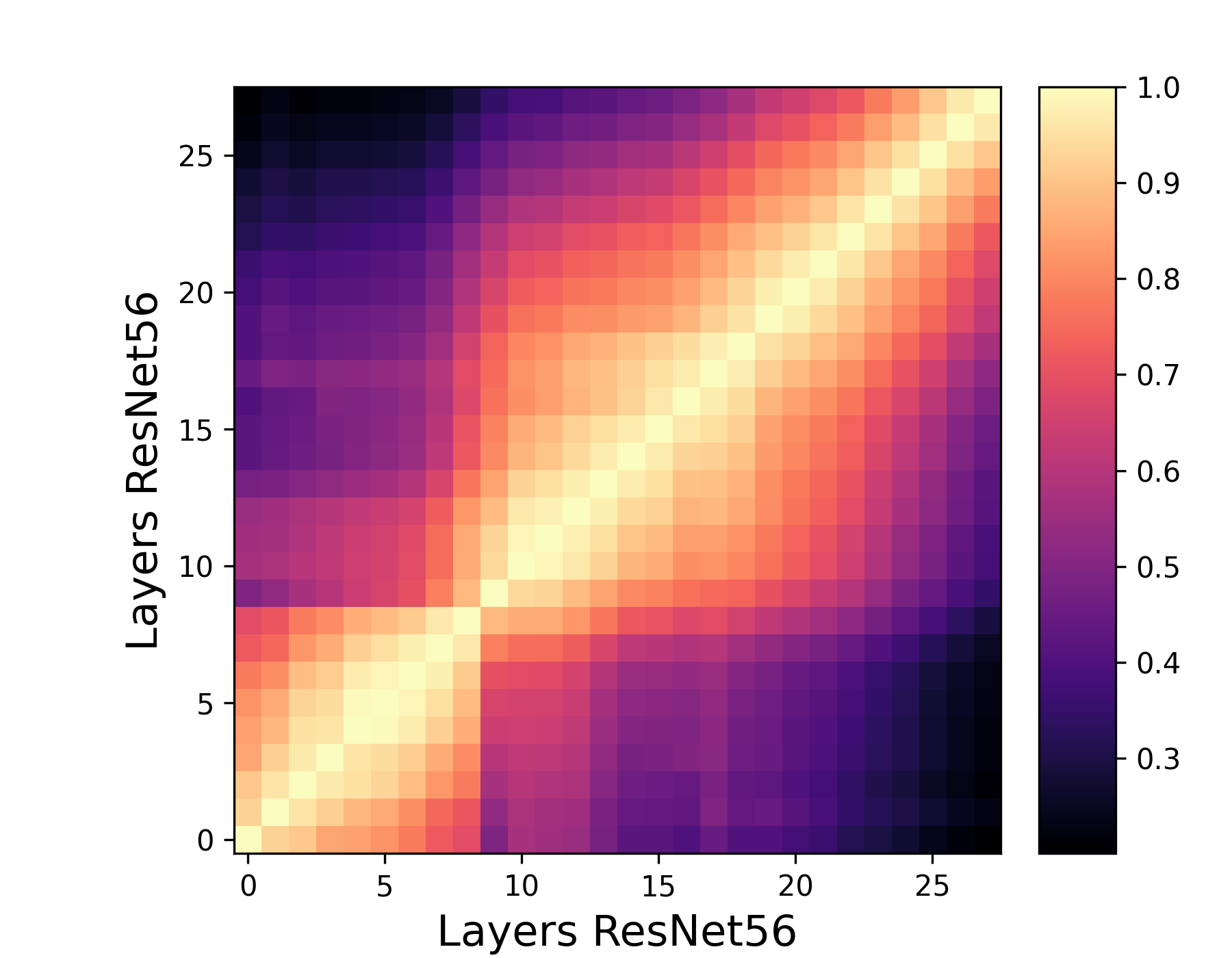}
			\includegraphics[width=\linewidth]{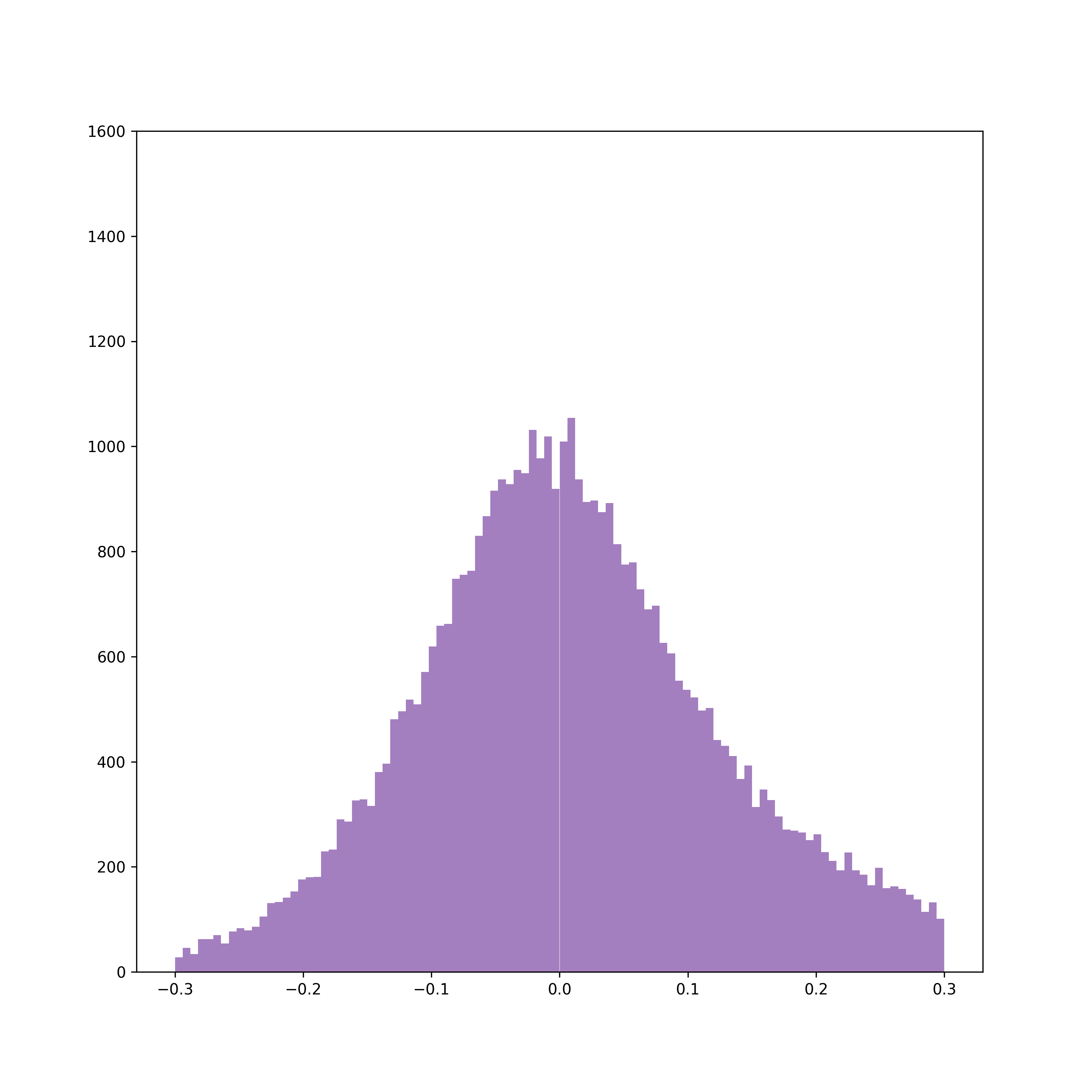}
			\label{fig:base1}
		\end{minipage}%
	}%
	\subfigure[$\beta=0.0005$]{
		\begin{minipage}[t]{0.24\linewidth}
			\centering
			\includegraphics[width=\linewidth]{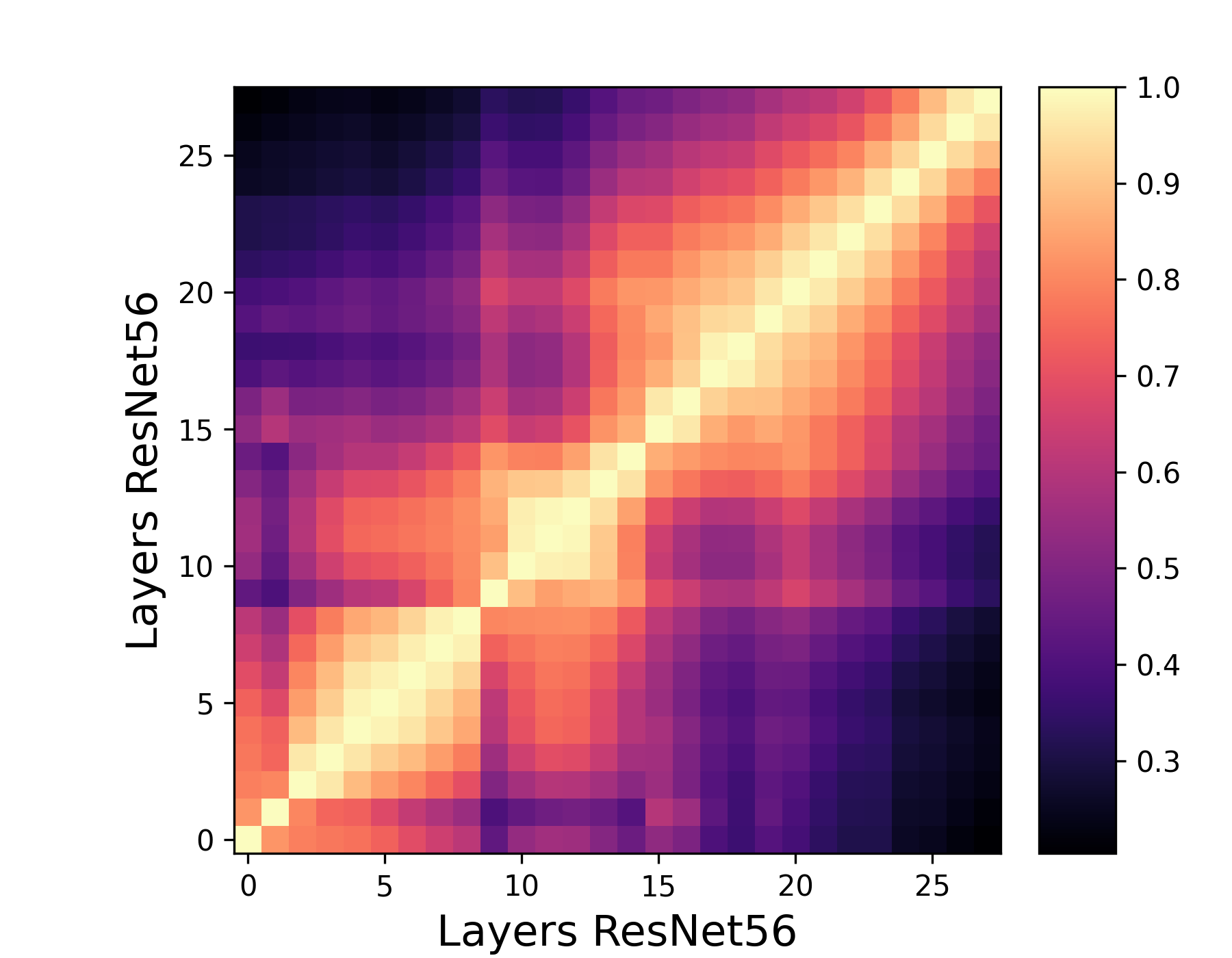}
			\includegraphics[width=\linewidth]{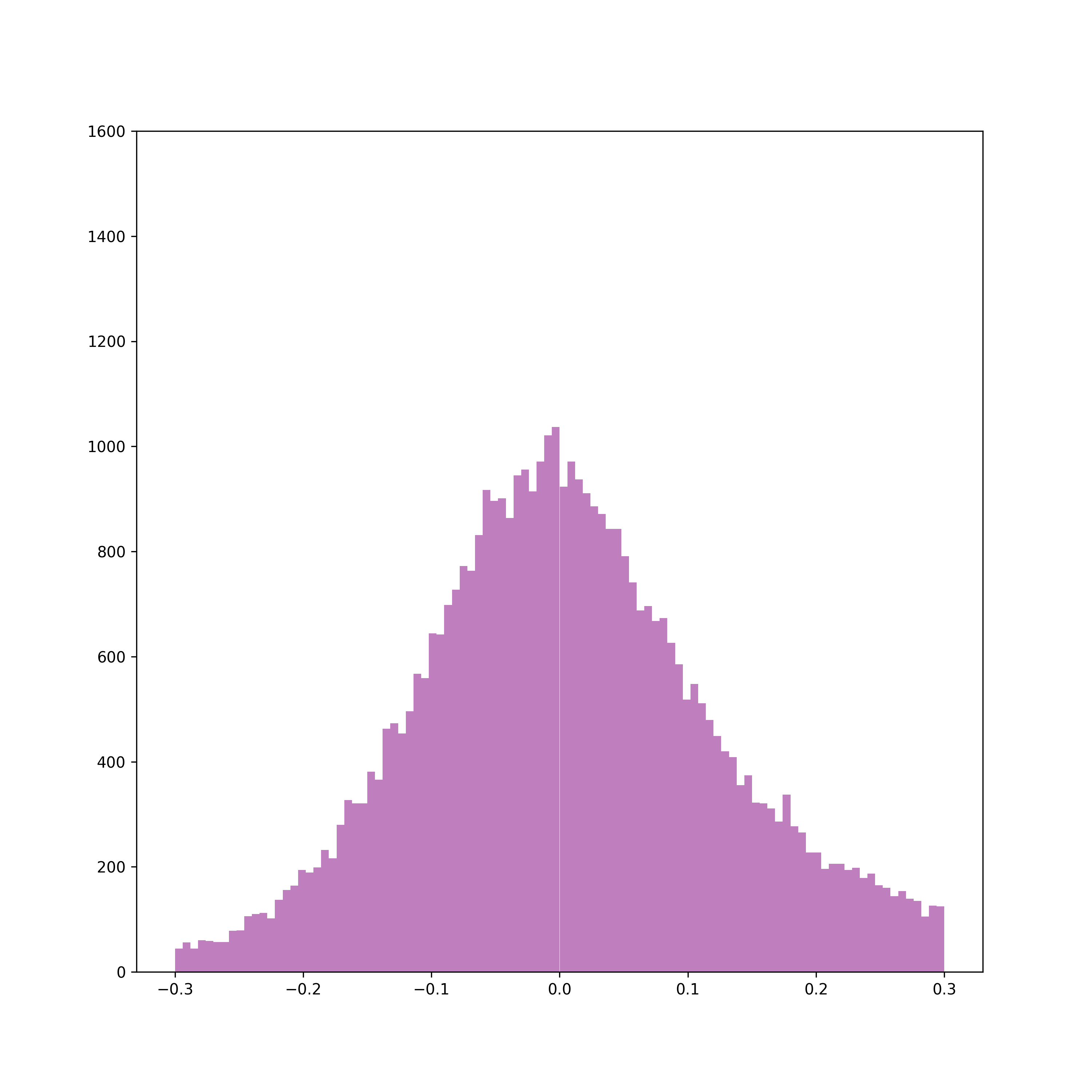}
			\label{fig:0.0005}
		\end{minipage}%
	}%
	\subfigure[$\beta=0.002$]{
		\begin{minipage}[t]{0.24\linewidth}
			\centering
			\includegraphics[width=\linewidth]{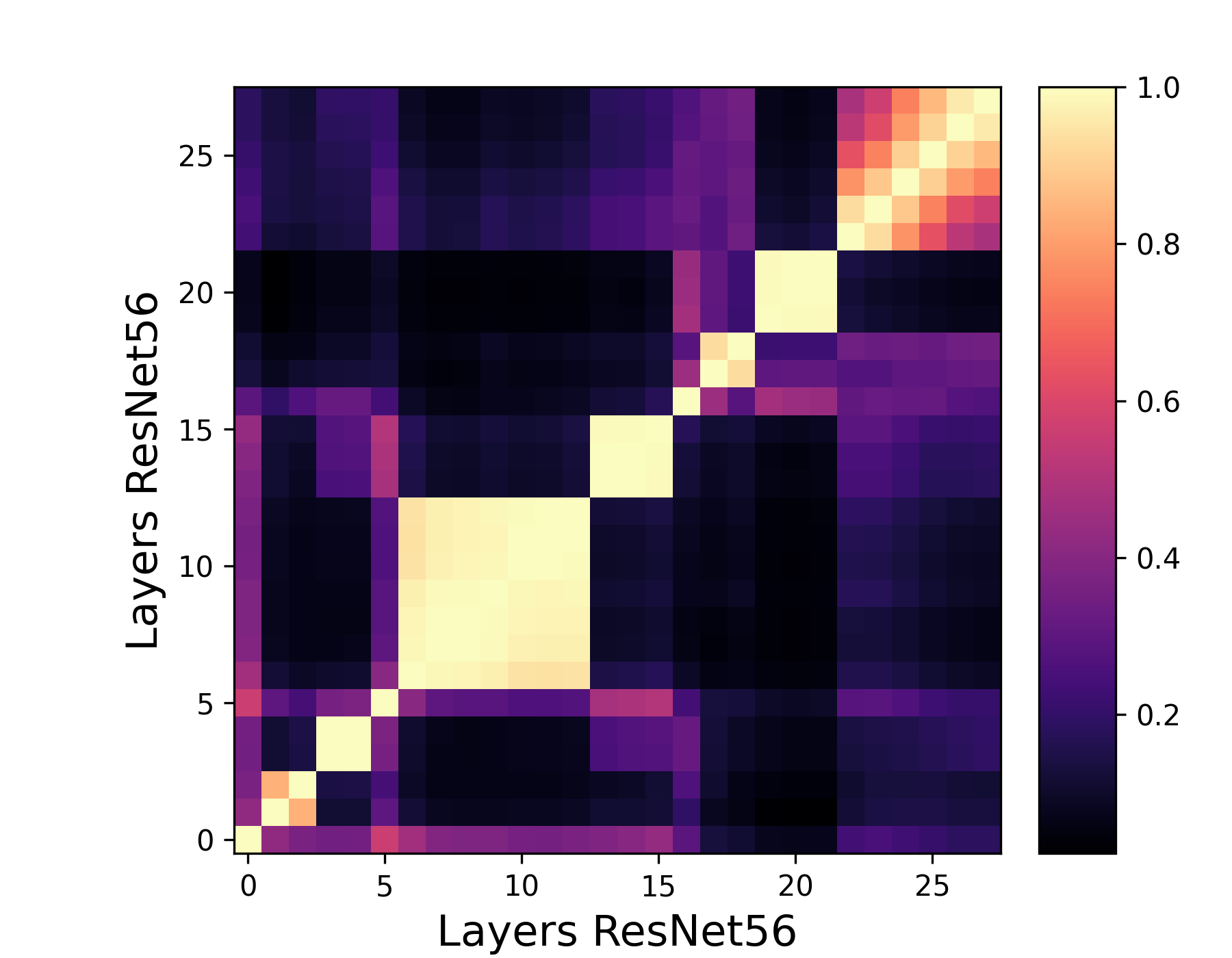}
			\includegraphics[width=\linewidth]{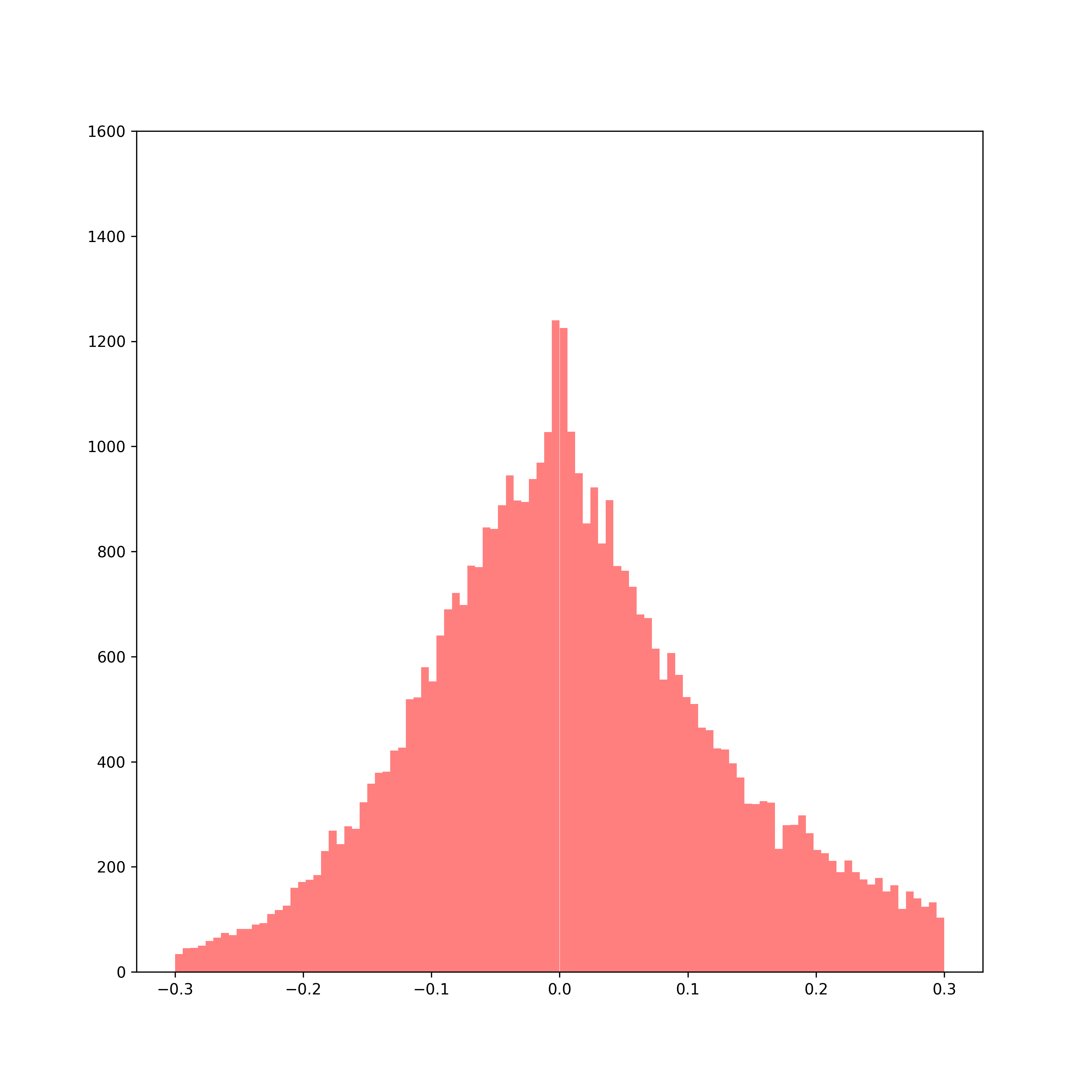}
			\label{fig:0.002}
		\end{minipage}%
	}%
	\subfigure[$\beta=0.005$]{
		\begin{minipage}[t]{0.24\linewidth}
			\centering
			\includegraphics[width=\linewidth]{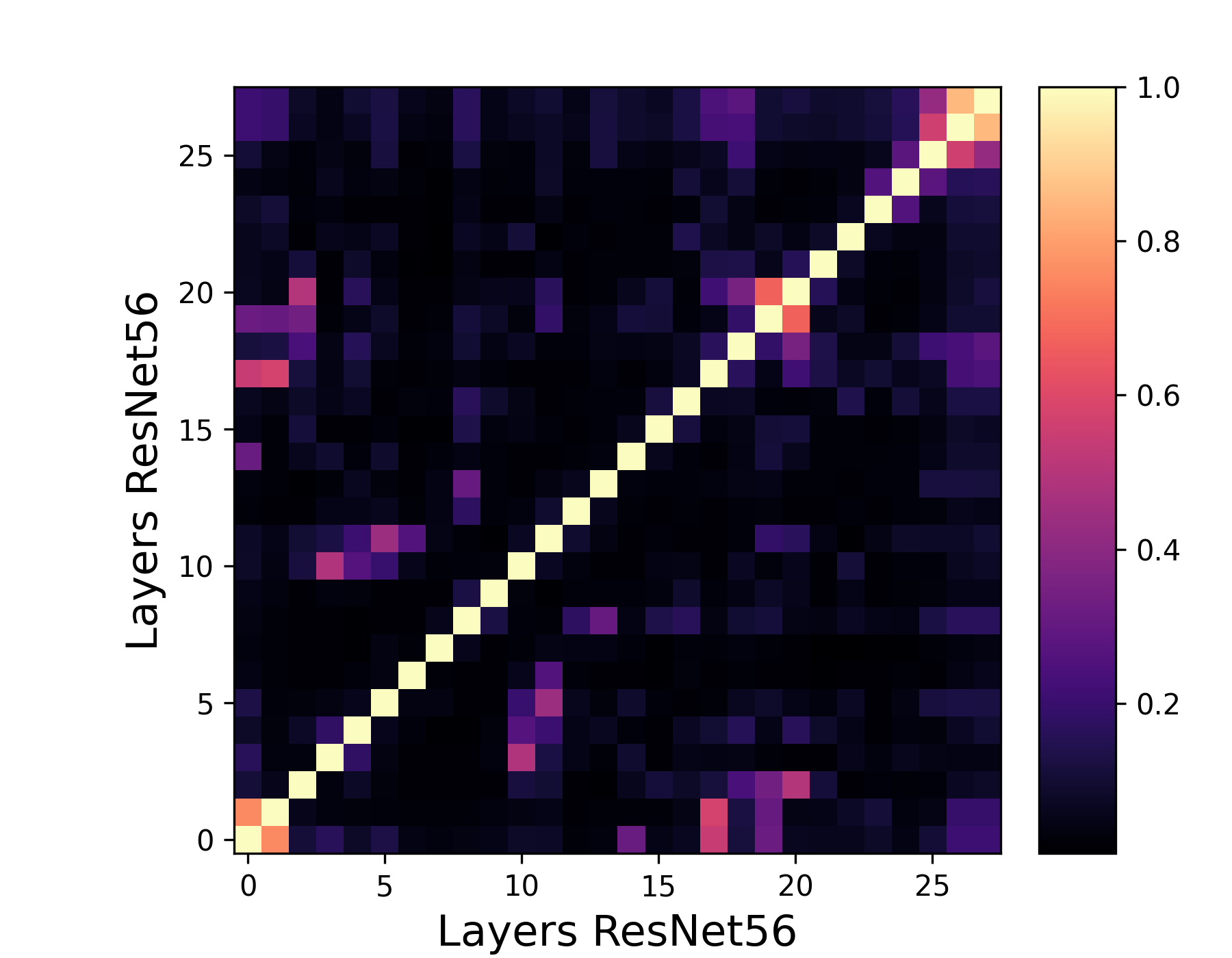}
			\includegraphics[width=\linewidth]{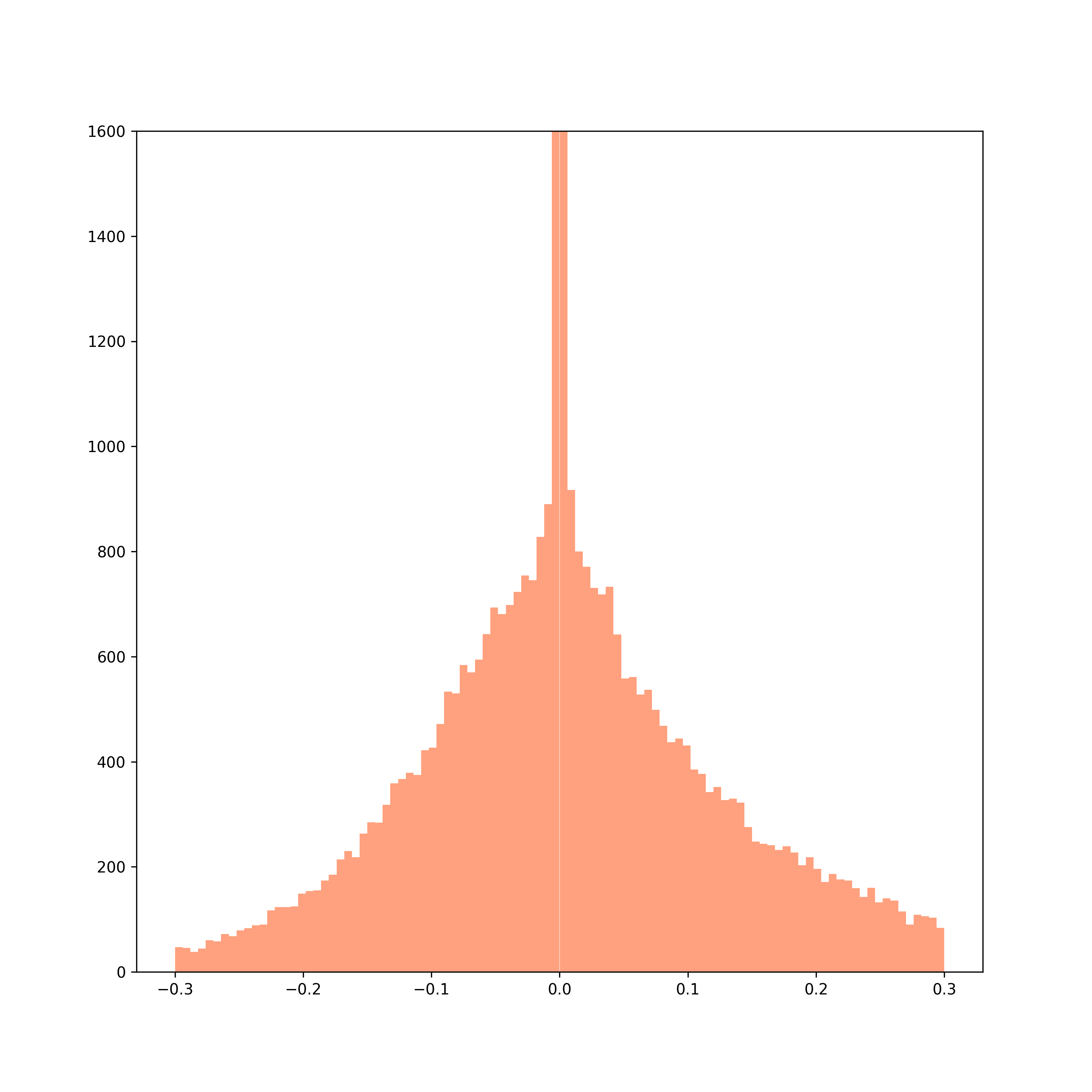}
			\label{fig:0.005}
		\end{minipage}%
	}%
	\caption{Visualizations of CKA similarity and corresponding parameter distribution (ResNet56)}
	\label{fig:visual1}
\end{figure}

\subsection{Dynamic token sparsification with CKA-SR}
To validate our CKA-SR method on vision transformers, we conduct dynamic token sparsification~\cite{rao2021dynamicvit} with our CKA-SR. Specifically, we calculate the CKA similarity between tokens produced by our CKA-SR model and introduce this similarity into the decision process of dynamic token sparsification. We rewrite the keeping probability of the $i^{th}$ token $\pi_{i,1}$ as 
\begin{align}
\label{Eqn:pi}
\pi'_{i,1}=\pi_{i,1}-\sum_{j \neq i}\mathbf{CKA_L}(X_i, X_j)
\end{align}
in which $j$ is all the other tokens except $i$. The whole algorithm is shown in Algorithm \ref{alg:algorithm}. And the results are shown in Table \ref{tab: DTS}.

\begin{algorithm}[tb]
\caption{Dynamic token sparsification with CKA-SR}
\label{alg:algorithm}
\textbf{Initialize}: ViT model pretrained with CKA-SR $f_W$, dataset $\mathcal{D}$, pruning ratio $r$ \\
\begin{algorithmic}[1] 
\STATE Calculate the feature map of each patch, and name it as a token.
\STATE Update each token through multiple transformer layers.
\WHILE{Numbers of remaining tokens $N_r$/Numbers of all tokens $N_t$ $>$ $r$}
\STATE Use the prediction module to calculate a mask for pruning, in which the keeping probability of the $i^{th}$ token equals $\pi'_{i,1}$ in Eq.~\eqref{Eqn:pi}.
\STATE Prune the tokens with the mask.
\STATE Update the remaining tokens through transformer layers.
\ENDWHILE
\STATE Conduct classification with the remaining few tokens.
\end{algorithmic}
\end{algorithm}

\begin{table}[H]
  \small
  \caption{The accuracy (\%) of DeiT-Tiny using CKA-SR for dynamic token sparsification on ImageNet}
  \label{tab: DTS}
  \centering
  \begin{adjustbox}{max width=\linewidth}
  \begin{tabular}{|l |l |l |l |l|}
    \hline
     Model & Hyperparameter $\rho$ & Params & FLOPs & Top-1 Accuracy \\
    \hline
     DeiT-Tiny & -- & 5.7M & 1.3G & 72.20 \\
    \hline
     Dynamic-DeiT-Tiny & 0.7 & 5.9M & 0.9G & 71.11  \\
     +CKA-SR & 0.7 & 5.9M & 0.9G & \textbf{71.35} \\
    \hline
     Dynamic-DeiT-Tiny & 0.9 & 5.9M & 1.2G & 72.31 \\
     +CKA-SR & 0.9 & 5.9M & 1.2G & \textbf{72.52} \\
    \hline
  \end{tabular}
  \end{adjustbox}
  \vspace{-10pt}
\end{table}

\subsection{Further ablation studies}

\subsubsection{Ablation study of samples and batches}
As our CKA-SR is proposed to explicitly reduce the interlayer similarity of network parameters instead of feature maps themselves, we could utilize several samples of each batch (\textit{generally 8 samples when the batch size is 128 or 256}) to compute CKA-SR. Besides, to reduce the expenses, we could calculate CKA-SR once out of several (\textit{generally 5 or 10}) batches. We conduct the ablation study of samples and batches with Random Sparse Training~\cite{liu2021unreasonable} method on CIFAR-10 dataset, and the sparsity of the ResNet20 models is 0.95. The results are shown in Table \ref{ablation_s&b}. \\

\begin{table}[H]
 \vspace{-10pt}
  \small
  \caption{Ablation study of samples and batches (\textit{Batch-m} means we calculate CKA-SR once out of m batches, and \textit{Sample-n} means we utilize n samples of each batch to calculate CKA-SR)}
  \label{ablation_s&b}
  \centering
  \begin{adjustbox}{max width=\linewidth}
  \begin{tabular}{|l |l |l |l |l |l |l |l |l |l |l |l |l |l |l |}
    \hline
     $\beta$ & Baseline & 1e-05 & 2e-05 & 5e-05 & 8e-05 & 1e-04 & 2e-04 & 5e-04 & 8e-04 & 1e-03 & 2e-03 & 5e-03 & 1e-02\\
    \hline
     Original CKA-SR & 84.16 & 84.69 & 84.38 & 84.42 & 84.45 & 84.11 & 84.40 & 84.39 & \textbf{85.03} & 84.82 & 84.08 & 83.86 & 84.03\\
    \hline
     Sample-16 and Batch-5 & 84.16 & 84.06 & 84.52 & 84.44 & \textbf{84.90} & 84.59 & 84.59 & 84.42 & 84.61 & 84.40 & 84.61 & 84.71 & 84.17\\
    \hline
     Sample-8 and Batch-5 & 84.16 & 84.19 & 84.51 & 84.44 & 84.45 & \textbf{84.97} & 83.88 & 84.68 & 84.41 & 84.56 & 84.52 & 84.00 & 84.41\\
    \hline
     Sample-4 and Batch-5 & 84.16 & 84.71 & 84.35 & 84.62 & 84.20 & 83.85 & 84.30 & 84.87 & \textbf{84.92} & 84.62 & 84.67 & 84.27 & 76.08\\
    \hline
     Sample-8 and Batch-10 & 84.16 & 84.28 & 84.27 & \textbf{84.62} & 84.14 & 84.27 & 84.06 & 84.25 & 84.17 & 84.19 & 84.59 & 84.43 & 84.08\\
    \hline
  \end{tabular}
  \end{adjustbox}
  \vspace{-10pt}
\end{table}

To conclude, we can just utilize several samples of each batch and calculate CKA-SR only once out of several batches. By using these implementations above, the performance is slightly worse than full CKA-SR (our original method which calculates CKA-SR for each sample in each batch), but it's still better than the baseline method. At the same time, by using these implementations, the expenses of calculating CKA-SR are greatly reduced.


\subsection{Visualizations of CKA similarity and parameter distribution}
As an extension of Figure 1 of our paper, we show the visualizations of CKA similarity and corresponding parameter distributions of our CKA-SR sparsity models, and compare them with baseline models. We conduct the visualization with Random Sparse Training~\cite{liu2021unreasonable} method on CIFAR-10 dataset, and the sparsity of ResNet56 and ResNet20 models is 0.95. The results are shown in Figure \ref{fig:visual1} and Figure \ref{fig:visual2} (Note that in Figure \ref{fig:0.05}, the Y-axis is up to 1750, and in Figure \ref{fig:0.08}, the Y-axis is up to 2800, which means the parameters are extremely concentrated around 0).

\begin{figure}[!tb]
	\centering
	\subfigure[Baseline]{
		\begin{minipage}[t]{0.24\linewidth}
			\centering
			\includegraphics[width=\linewidth]{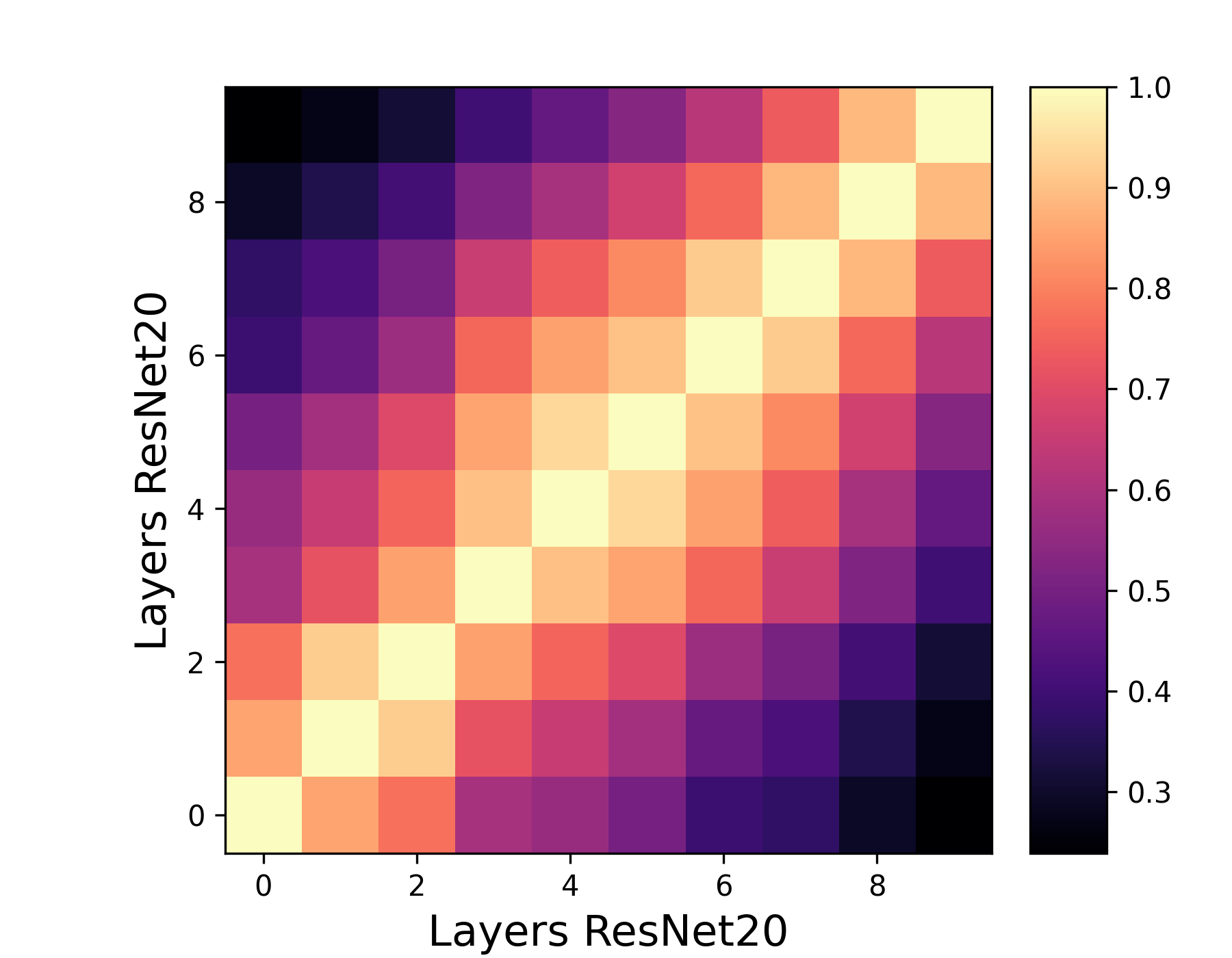}
			\includegraphics[width=\linewidth]{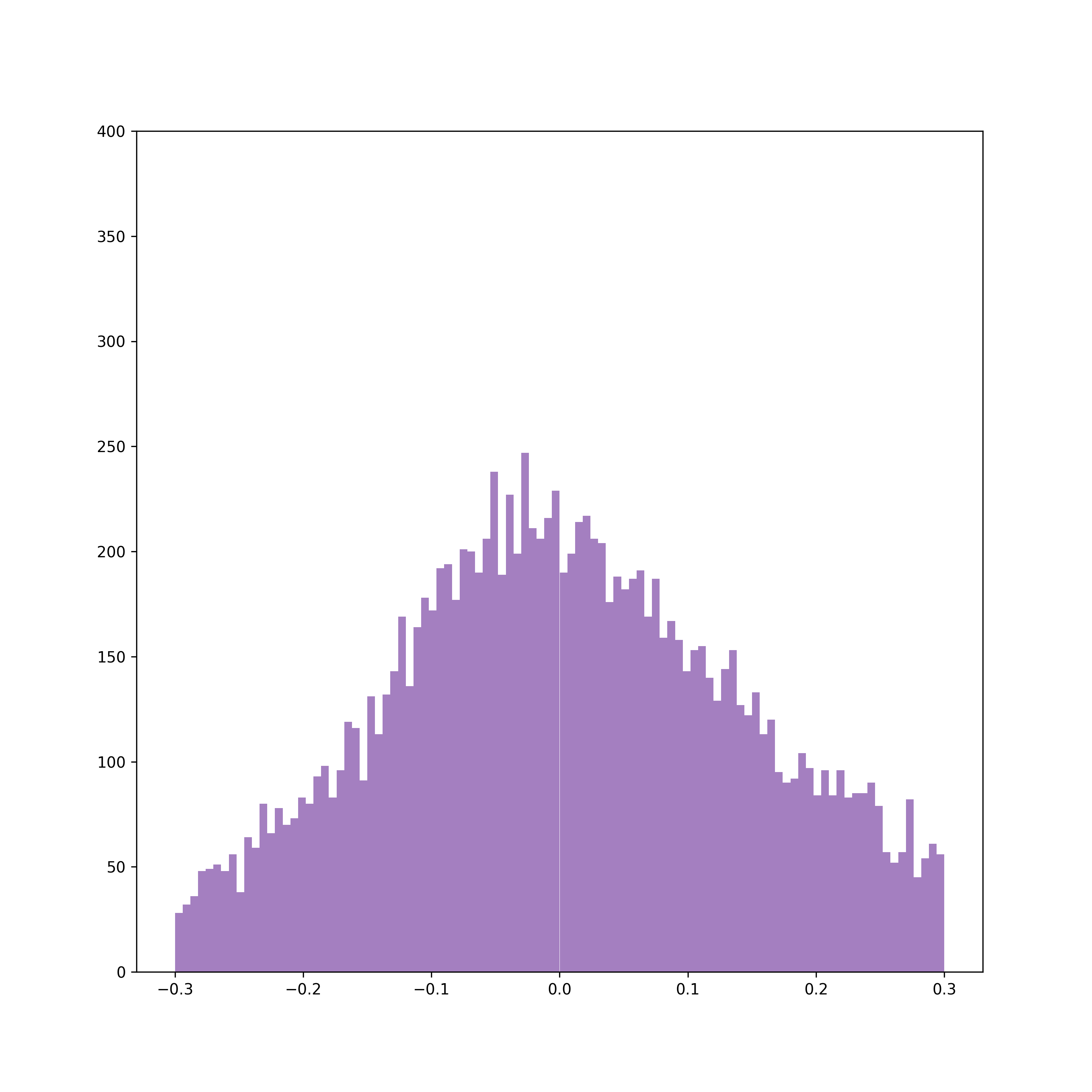}
			\label{fig:base2}
		\end{minipage}%
	}%
	\subfigure[$\beta=0.01$]{
		\begin{minipage}[t]{0.24\linewidth}
			\centering
			\includegraphics[width=\linewidth]{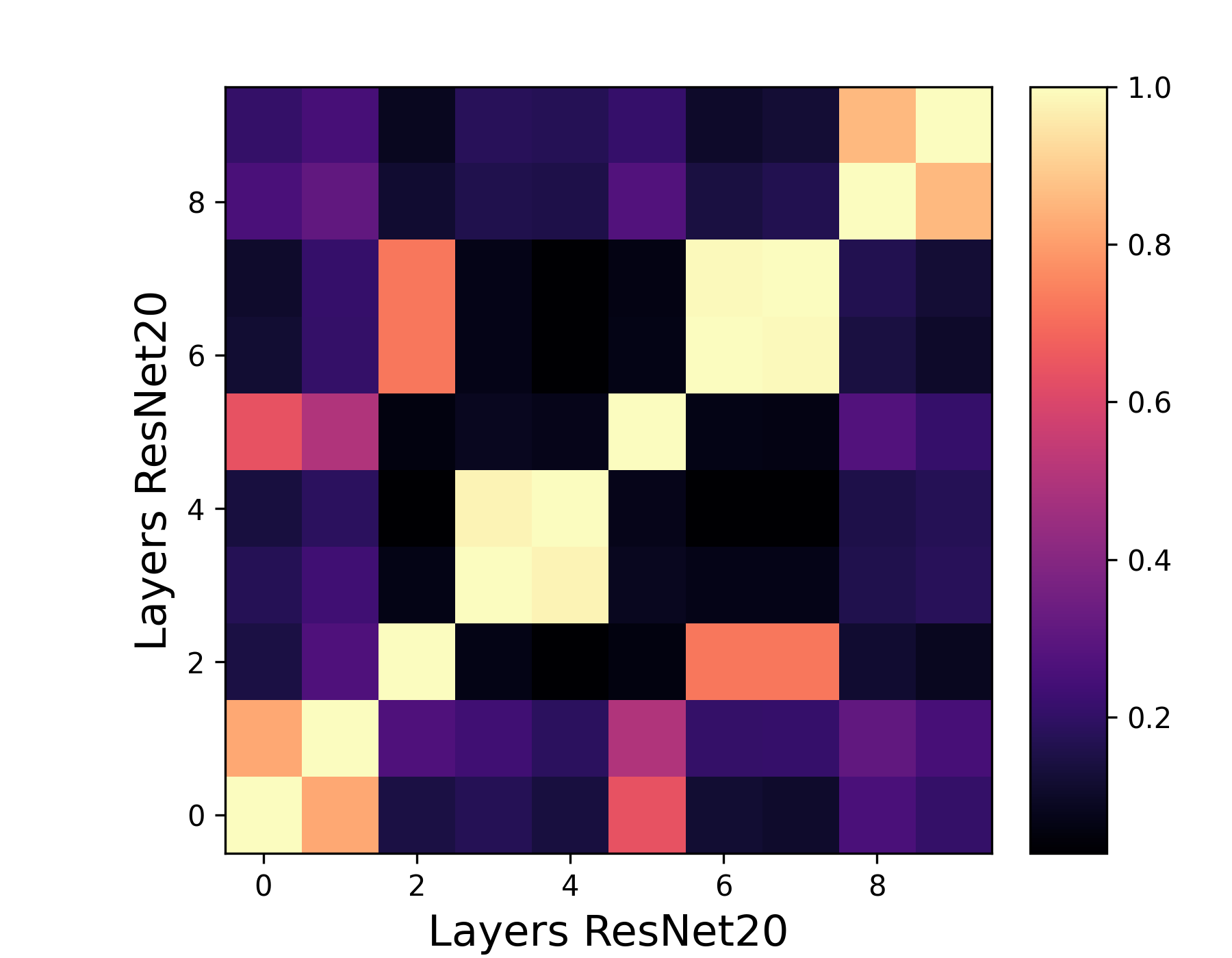}
			\includegraphics[width=\linewidth]{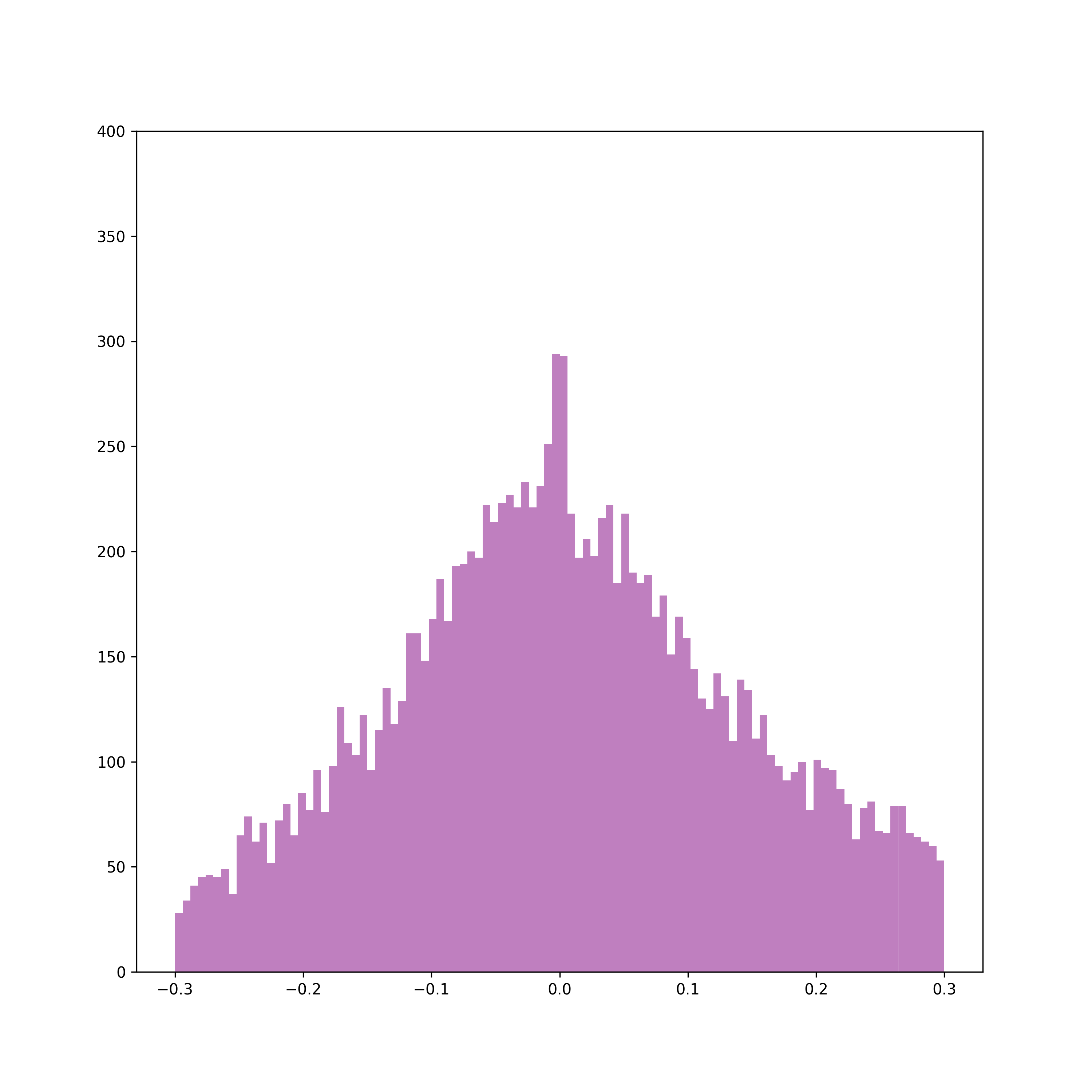}
			\label{fig:0.01}
		\end{minipage}%
	}%
	\subfigure[$\beta=0.05$]{
		\begin{minipage}[t]{0.24\linewidth}
			\centering
			\includegraphics[width=\linewidth]{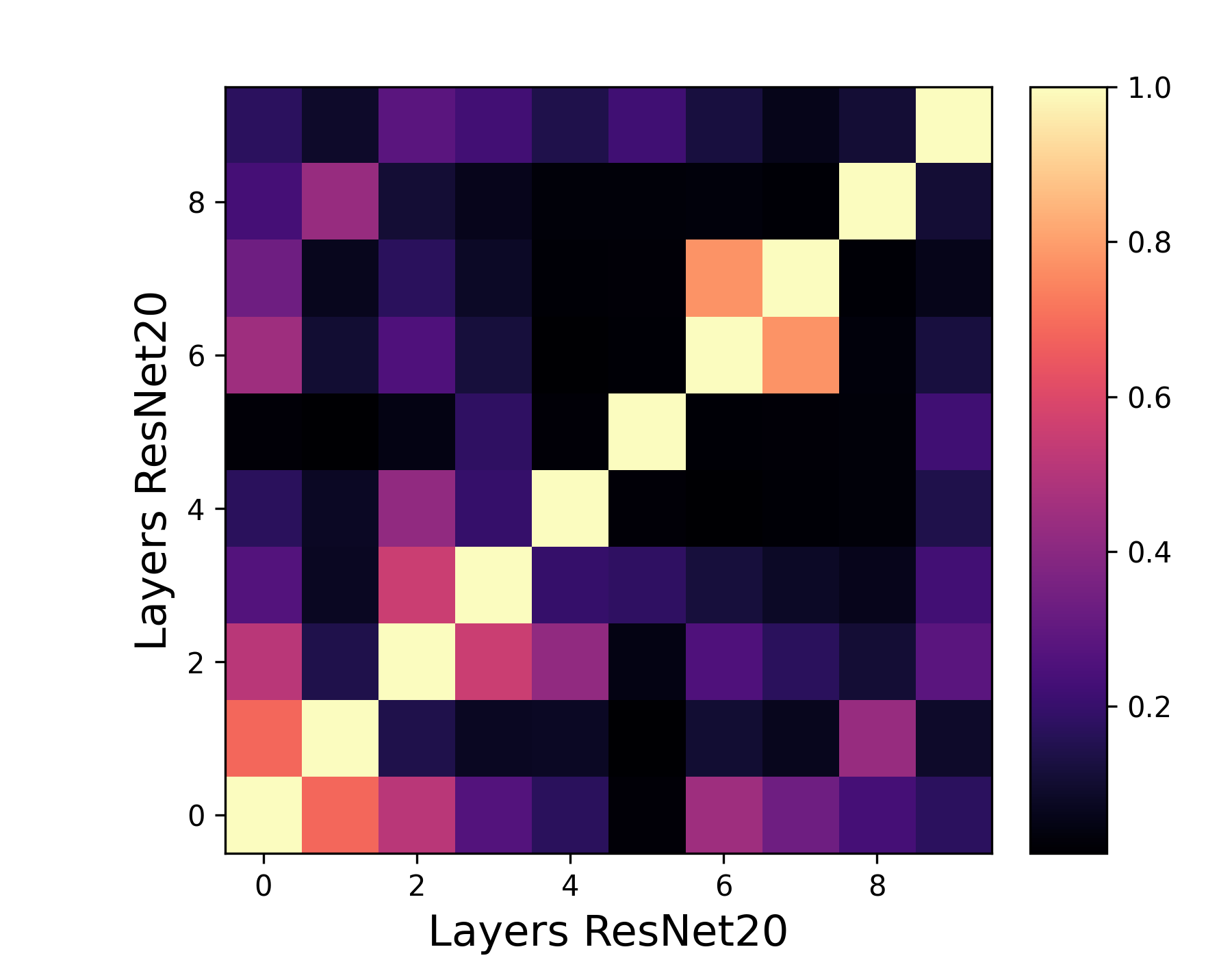}
			\includegraphics[width=\linewidth]{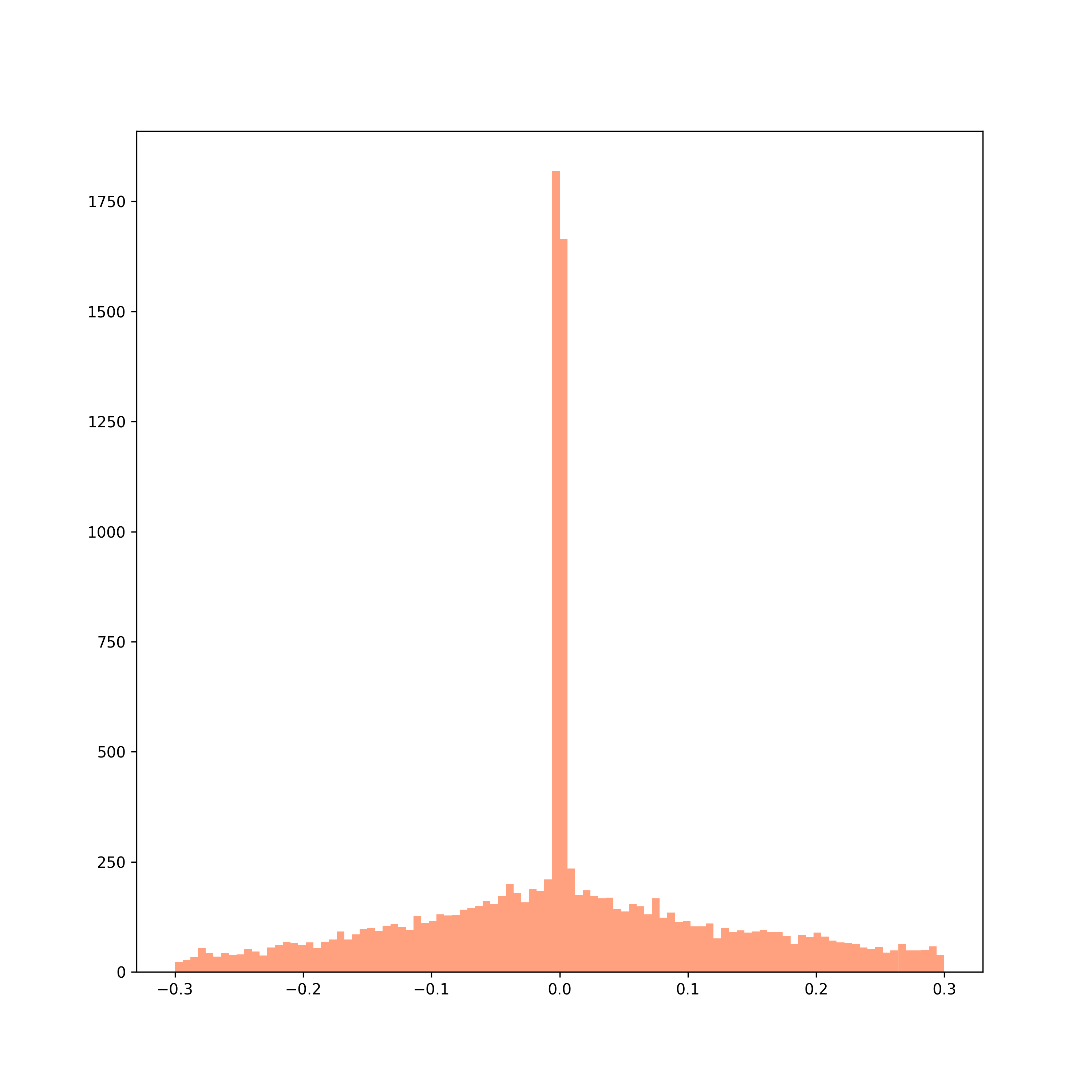}
			\label{fig:0.05}
		\end{minipage}%
	}%
	\subfigure[$\beta=0.08$]{
		\begin{minipage}[t]{0.24\linewidth}
			\centering
			\includegraphics[width=\linewidth]{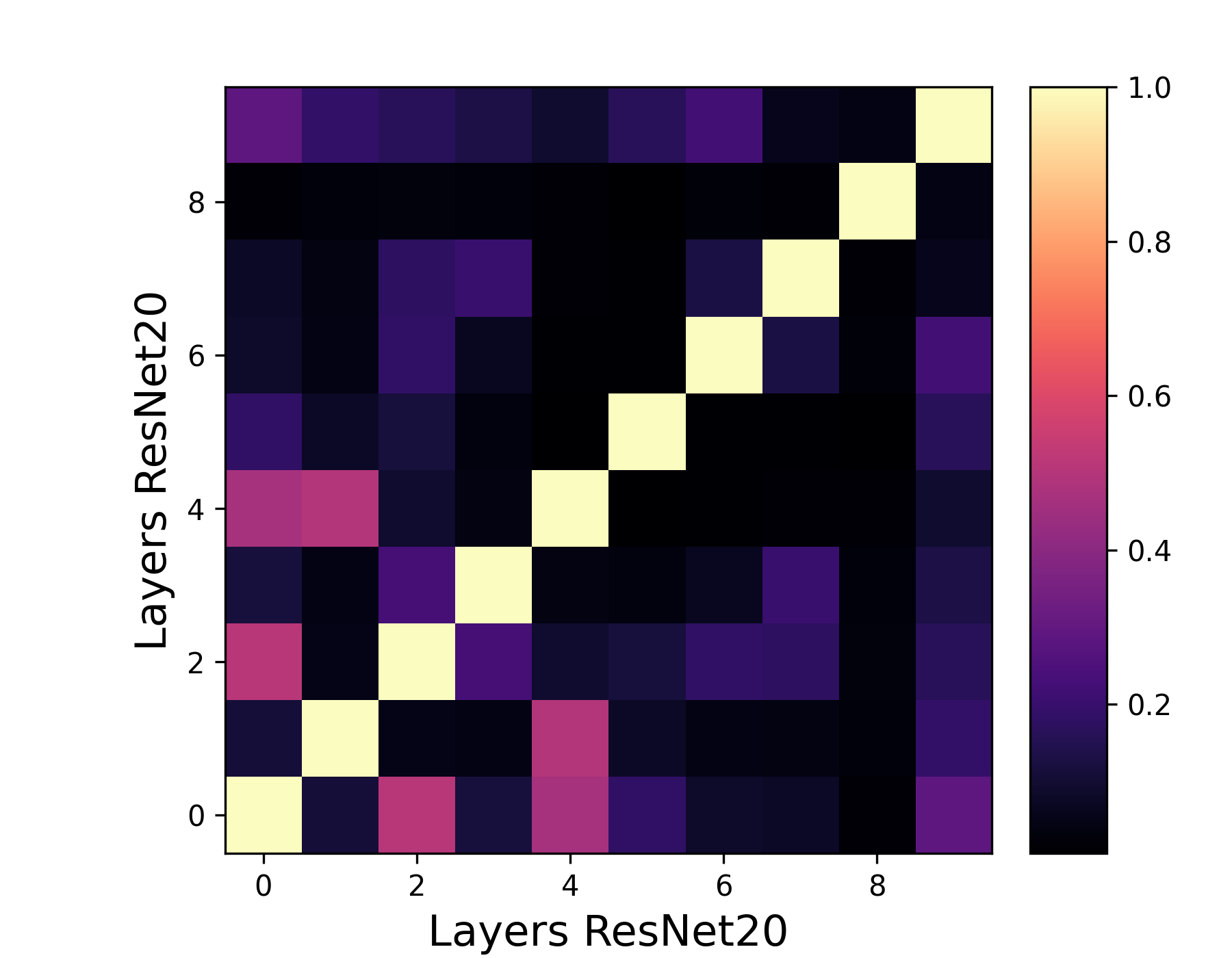}
			\includegraphics[width=\linewidth]{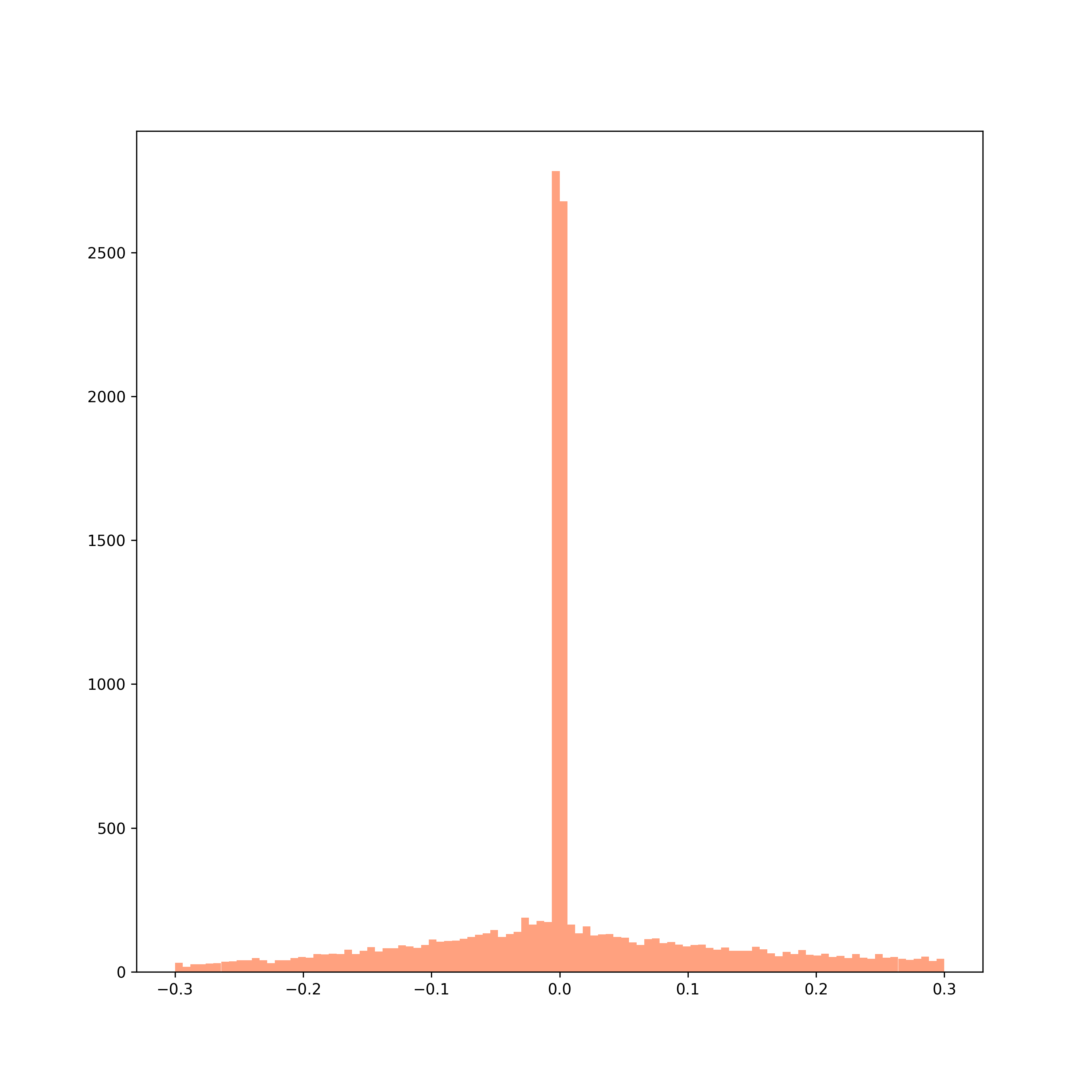}
			\label{fig:0.08}
		\end{minipage}%
	}%
	\caption{Visualizations of CKA similarity and corresponding parameter distribution (ResNet20)}
	\label{fig:visual2}
\end{figure}

We can conclude that with the increase of hyperparameter $\beta$, the interlayer feature similarity decreases and the network sparsity increases.


\subsection{Comparison with other similarity measurements}
To compare the performance of CKA and CCA in sparse training, we introduce CCA into sparse training. We find that CCA also increases the performance of sparse training, but the increment is smaller than CKA-SR. Besides, CCA is less robust than CKA, which means sparse training with CCA is harder to converge than CKA-SR. 

We conduct the experiments with Random Sparse Training~\cite{liu2021unreasonable} on the CIFAR-10 dataset, and the sparsity is 0.95. The results are shown in Table \ref{tab: compCCA}.\\

\begin{table}[H]
 \vspace{-10pt}
  \small
  \caption{Comparison of CKA-SR and CCA}
  \label{tab: compCCA}
  \centering
  \begin{adjustbox}{max width=\linewidth}
  \begin{tabular}{|l |c |c |c |}
    \hline
     Settings & Baseline & CCA & CKA-SR\\
    \hline
     Top1-Acc & 84.16 & 84.58 & 85.03\\
    \hline
  \end{tabular}
  \end{adjustbox}
  \vspace{-10pt}
\end{table}

\subsection{Models pre-trained with CKA-SR}
As an additional experimental result, we show the performance increment of models pre-trained with CKA-SR in Table \ref{tab: premodels}. Besides, we show the parameter distribution of these models in Figure \ref{fig: premodels} to prove that our CKA-SR increases the sparsity of the pre-trained models.

\begin{table}[H]
  \small
  \caption{The accuracy (\%) when plugging CKA-SR to pre-training of ResNet18 and ResNet50 models on CIFAR100 and ImageNet}
  \label{tab: premodels}
  \centering
  \begin{adjustbox}{max width=\linewidth}
  \begin{tabular}{|l |l |l |l |}
    \hline
     \multirow{2}{*}{Backbone} & \multirow{2}{*}{Methods} &\multicolumn{2}{c|}{Dataset} \\
     \cline{3-4}
       &  & CIFAR100 & ImageNet \\
     \hline
    \multirow{2}{*}{ResNet18} & Baseline & 75.61 & 69.62\\
     & +CKA-SR & \textbf{76.32} & \textbf{69.90}\\
     \hline
    \multirow{2}{*}{ResNet50} & Baseline & 77.39 & 76.15\\
     & +CKA-SR & \textbf{78.94} & \textbf{76.39}\\
    \hline
  \end{tabular}
  \end{adjustbox}
  \vspace{-10pt}
\end{table}

\begin{figure}[H]
	\centering
	\subfigure[Baseline ResNet18]{
		\begin{minipage}[t]{0.24\linewidth}
			\centering
			\includegraphics[width=\linewidth]{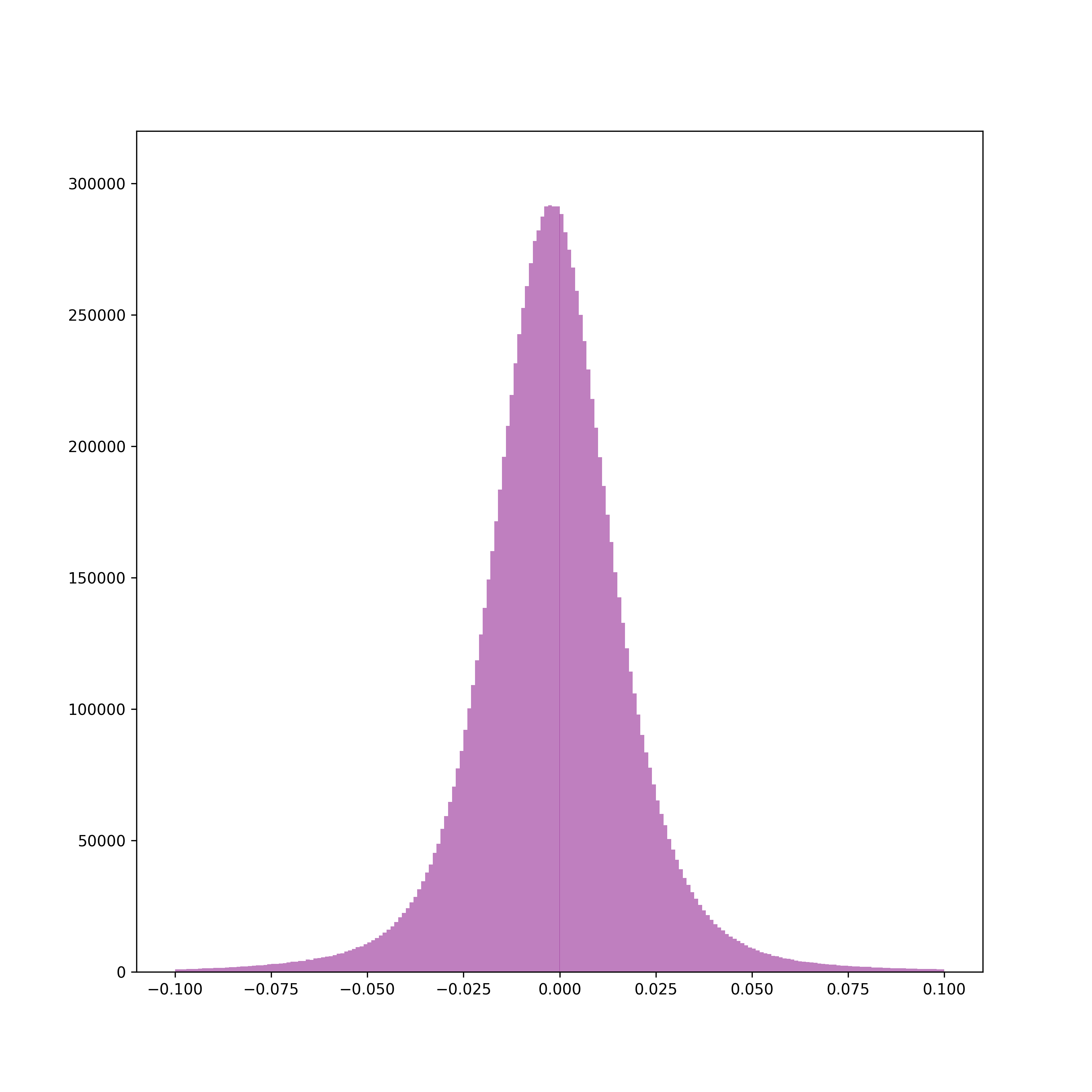}
			\label{fig:18base}
		\end{minipage}%
	}%
	\subfigure[CKA-SR ResNet18]{
		\begin{minipage}[t]{0.24\linewidth}
			\centering
			\includegraphics[width=\linewidth]{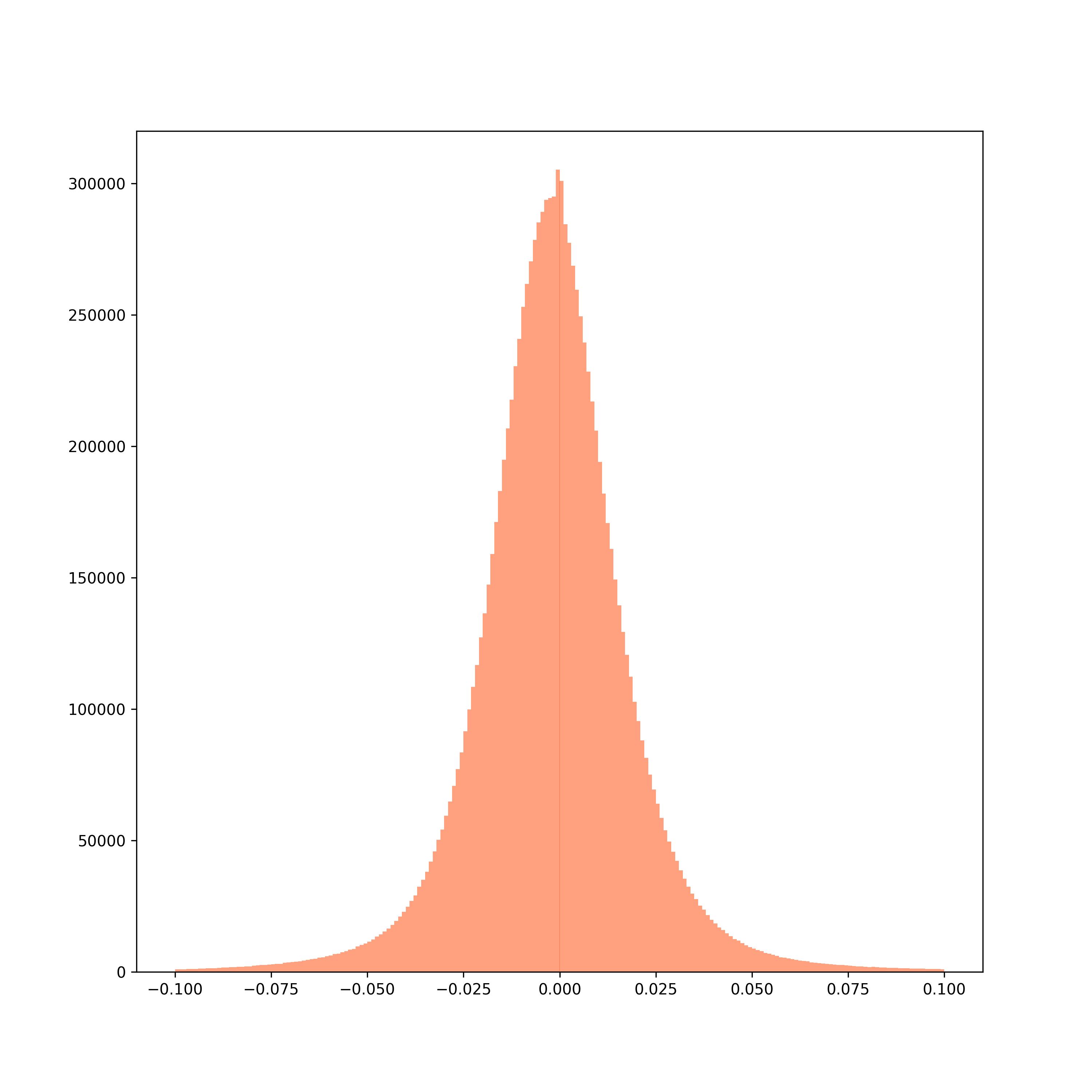}
			\label{fig:18cka}
		\end{minipage}%
	}%
	\subfigure[Baseline ResNet50]{
		\begin{minipage}[t]{0.24\linewidth}
			\centering
			\includegraphics[width=\linewidth]{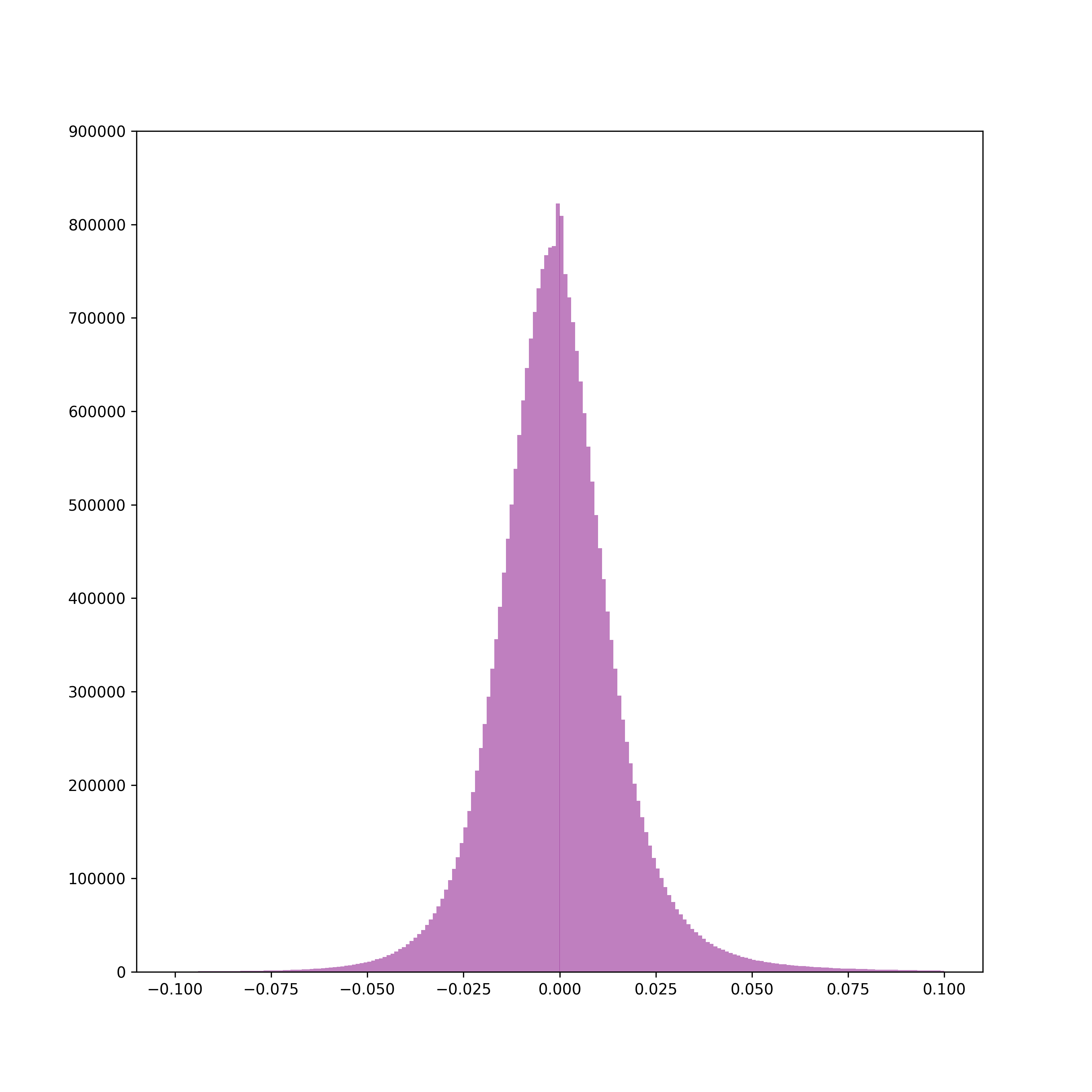}
			\label{fig:50base}
		\end{minipage}%
	}%
	\subfigure[CKA-SR ResNet50]{
		\begin{minipage}[t]{0.24\linewidth}
			\centering
			\includegraphics[width=\linewidth]{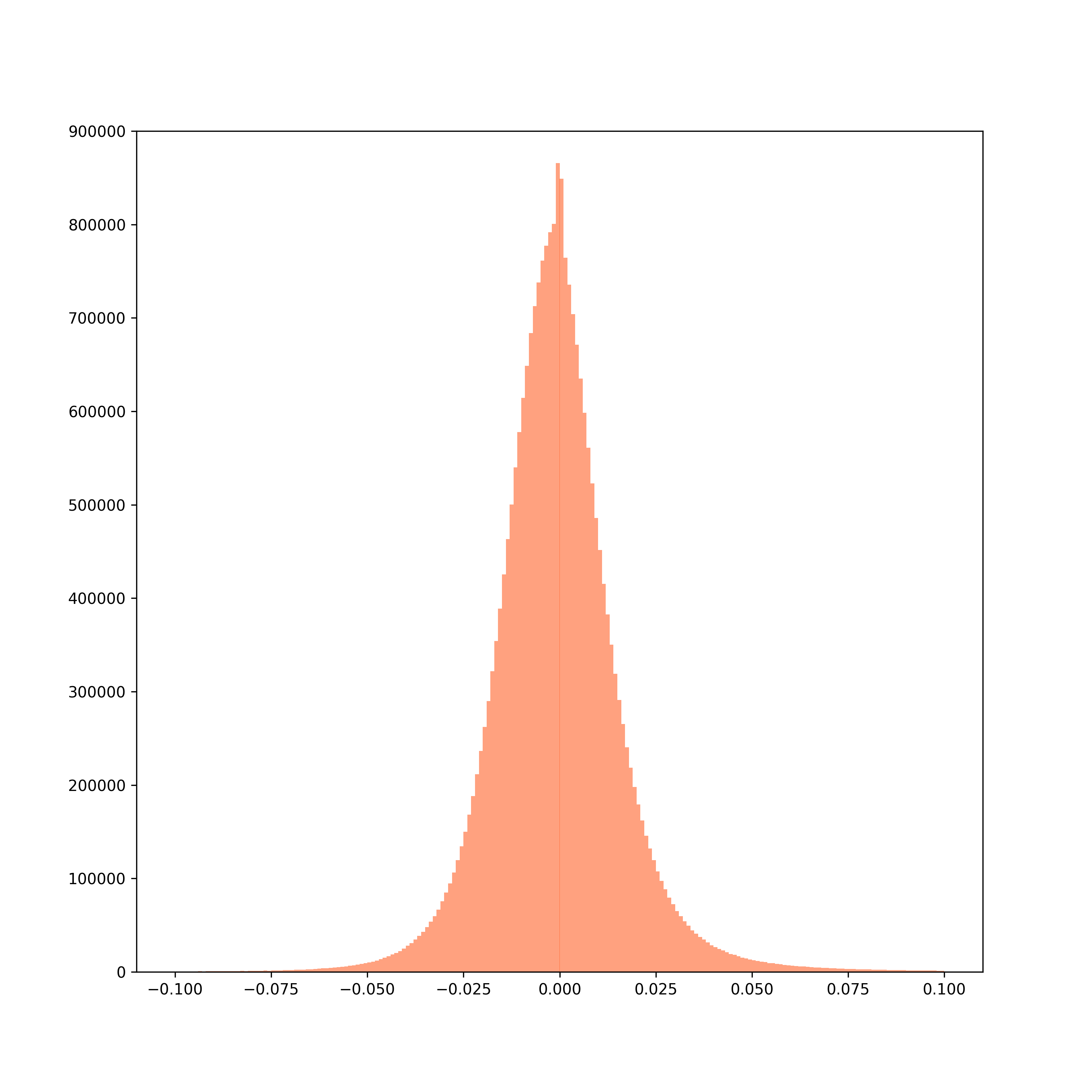}
			\label{fig:50cka}
		\end{minipage}%
	}%
	\caption{Parameter distributions of baseline and CKA-SR models (ResNet18 and ResNet50 models pre-trained on ImageNet)}
	\label{fig: premodels}
	\vspace{-10pt}
\end{figure}

It is concluded from Figure \ref{fig: premodels} and Table \ref{tab: premodels} that our CKA-SR increases both the sparsity (by making parameters more concentrated around 0) and the classification accuracy of ResNet18 and ResNet50 models trained on ImageNet.

\section{\textbf{Aug}CKA-SR}
\subsection{Definition}
\textbf{Aug}CKA-SR is an improved version of CKA-SR. We introduce the equation of \textbf{Aug}CKA-SR as follows, in which we use features with different shapes to calculate CKA-SR. 
\begin{equation}
    \label{eqn:augcka-sr}
    \begin{aligned}
    &\begin{aligned}
    \mathcal{L} &= \mathcal{L_{\mathcal{E}}} + \mathcal{L_{\mathcal{C}}} \\
    &= \mathcal{L_{\mathcal{E}}} + \beta \cdot \frac{2}{N(N-1)} \cdot \sum_{i=0}^{N-1} \sum_{j=i+1}^{N} \mathbf{CKA}_{Linear}(X_i,X_j)
    \end{aligned}\\
    \end{aligned}
\end{equation}
In Eq.~\eqref{eqn:augcka-sr}, $N$ is the total number of layers. In this Augmented CKA-SR, we calculate the similarity between each pair of different layers and utilize the average similarity in the loss function.

Then we provide the implementation of \textbf{Aug}CKA-SR below. 

\lstset{
  language=Python,
  basicstyle=\small\ttfamily,
  keywordstyle=\color{blue},
  stringstyle=\color{red},
  commentstyle=\color{green},
  numbers=right,
  numberstyle=\tiny,
  numbersep=5pt,
  breaklines=true,
  showstringspaces=false
}

\begin{lstlisting}[language=Python]
def centering(K):
    n = K.size(0)
    unit = torch.ones([n, n]).cuda()
    I = torch.eye(n).cuda()
    H = I - unit / n
    return torch.mm(torch.mm(H, K), H)

def linear_HSIC(X, Y):
    L_X = torch.mm(X, X.T)
    L_Y = torch.mm(Y, Y.T)
    return torch.sum(centering(L_X) * centering(L_Y))

def linear_CKA(X, Y):
    hsic = linear_HSIC(X, Y)
    var1 = torch.sqrt(linear_HSIC(X, X))
    var2 = torch.sqrt(linear_HSIC(Y, Y))
    return hsic / (var1 * var2)

class CKALoss(nn.CrossEntropyLoss):
    def __init__(self, ):
        super(CKALoss, self).__init__()

    def forward(self, blocks):
        blocklist = []
        losssum = torch.zeros(1).cuda()
        for i in range(len(blocks)):
            for j in range(len(blocks[i])):
                blocklist.append(blocks[i][j])
        # Calculate CKA similarity between each two layers
        for i in range(len(blocklist)-1):
            for j in range(i+1, len(blocklist)):
                X = blocklist[i].flatten(1).float()
                Y = blocklist[j].flatten(1).float()
                losssum += linear_CKA(X, Y)
        loss = 2*losssum/len(blocklist)/(len(blocklist)-1)
        return loss
\end{lstlisting}

\subsection{Experimental results}
We conduct experiments on our \textbf{Aug}CKA-SR. We plug our \textbf{Aug}CKA-SR into the training process of Random Sparse Training method. We adopt ResNet20 as the backbone, and apply sparsity ratios from 0.70 to 0.998 for fair comparisons. The results are in Table~\ref{tab:sparse}. 
\begin{table}[H]
  \small
  \caption{The accuracy (\%) when plugging CKA-SR and \textbf{Aug}CKA-SR to Random Sparse Training(RST) method on CIFAR-100 from scratch.}
  \label{tab:sparse}
  \centering
  \begin{adjustbox}{max width=\linewidth}
  \begin{tabular}{|l|l|l|l|l|l|l|l|}
    \hline
     \cellcolor{white}\multirow{2}{*}{Backbone} & \cellcolor{white}\multirow{2}{*}{Method} & \multicolumn{6}{c|}{Sparsity} \\
    \cline{3-8}
       &  & \textit{0.70} & \textit{0.85} & \textit{0.90} & \textit{0.95} & \textit{0.98} & \textit{0.998}\\
     \hline
    \multirow{3}{*}{ResNet20} & Random\cite{liu2021unreasonable} & 65.42 & 60.37 & 56.96 & 47.27 & 33.74 & 2.95 \\
     & +CKA-SR & 65.60 & 60.86 & \textbf{57.25} & 48.26 & 34.44 &\textbf{3.32}  \\
     & +\textbf{Aug}CKA-SR & \textbf{65.85} & \textbf{60.89} & 57.02 & \textbf{48.39} & \textbf{34.60} & 3.04  \\
    \hline
  \end{tabular}
  \end{adjustbox}
  \vspace{-10pt}
\end{table}

Table~\ref{tab:sparse} shows that \textbf{Aug}CKA-SR is effective in Random Sparse Training and outperforms original CKA-SR in most sparsity levels, which fully demonstrates the effectiveness of the improvement proposed by us.

\renewcommand\bibname{References}
\makeatletter
\renewenvironment{thebibliography}[1]
  {\section*{\bibname}%
   \@mkboth{\MakeUppercase\bibname}{\MakeUppercase\bibname}%
   \list{\@biblabel{\@arabic\c@enumiv}}%
        {\settowidth\labelwidth{\@biblabel{#1}}%
         \leftmargin\labelwidth
         \advance\leftmargin\labelsep
         \@openbib@code
         \usecounter{enumiv}%
         \let\p@enumiv\@empty
         \renewcommand\theenumiv{\@arabic\c@enumiv}}%
   \sloppy
   \clubpenalty4000
   \@clubpenalty \clubpenalty
   \widowpenalty4000%
   \sfcode`\.\@m}
  {\def\@noitemerr
    {\@latex@warning{Empty `thebibliography' environment}}%
   \endlist}
\makeatother
\bibliography{references}

%% file: Tex/01_Introduction.tex
\section{Introduction}
\label{sec:intro}

Deep Neural Networks (DNNs) achieve great success on many important tasks, including but not limited to computer vision and natural language processing. Such accurate solutions highly rely on overparameterization, which results in a tremendous waste of resources. A variety of methods are proposed to solve such issues, including model pruning~\cite{han2015deep, han2015learning, he2019filter} and sparse training~\cite{bai2022dual, frankle2018lottery, liu2021unreasonable, wang2020picking}. Sparse training aims to train a sparse network from scratch, which reduces both training and inference expenses.

A recent study~\cite{nguyen2020wide} shows the close relation between overparameterization and interlayer feature similarity (\emph{i.e.} similarity between features of different layers, as shown in Figure \ref{fig:base} ). Specifically, overparameterized models possess obviously greater similarity between features of different layers. Concluding from the facts above, we know that both interlayer feature similarity and network sparsity are deeply related to overparameterization. Inspired by this, we utilize the interlayer feature similarity to increase network sparsity and preserve accuracy at a high level, namely by adopting similarity methods to solve sparsity problems.

Following this path, we survey similarity measurements of features, including Canonical Correlation Analysis (CCA)~\cite{hardoon2004canonical, hotelling1992relations, ramsay1984matrix} and Centered Kernel Alignment (Linear-CKA and RBF-CKA)~\cite{kornblith2019similarity}, etc. Among these measurements, CKA measurement is advanced and robust, for it reliably identifies correspondences between representations in networks with different widths trained from different initializations. Theoretically, CKA measurement has many good properties, including invariance to orthogonal transform and isotropic scaling, and close correlation with mutual information~\cite{zheng2021information}. The advantages of CKA make it possible to propose robust methods to solve sparsity problems with interlayer feature similarity. 

To this end, we propose \textbf{CKA}-based \textbf{S}parsity \textbf{R}egularization (CKA-SR) by introducing the CKA measurement into training loss as a regularization term, which is a plug-and-play term and forces the reduction of interlayer feature similarity. Besides, we further prove that the proposed CKA-SR increases the sparsity of the network by using information bottleneck(IB) theory~\cite{saxe2019information, tishby2000information, tishby2015deep, zheng2021information}. Specifically, we mathematically prove that our CKA-SR reduces the mutual information between the features of the intermediate and input layer, which is one of the optimization objectives of the information bottleneck method. Further, we prove that reducing the mutual information above is equivalent to increasing network sparsity. By these proofs, we demonstrate the equivalence of reducing interlayer feature similarity and increasing network sparsity, which heuristically investigates the intrinsic link between interlayer feature similarity and network sparsity.

\begin{figure*}[!tb]
	\centering
	\subfigure[w/o CKA-SR]{
		\begin{minipage}[t]{0.3\linewidth}
			\centering
			\includegraphics[width=\linewidth]{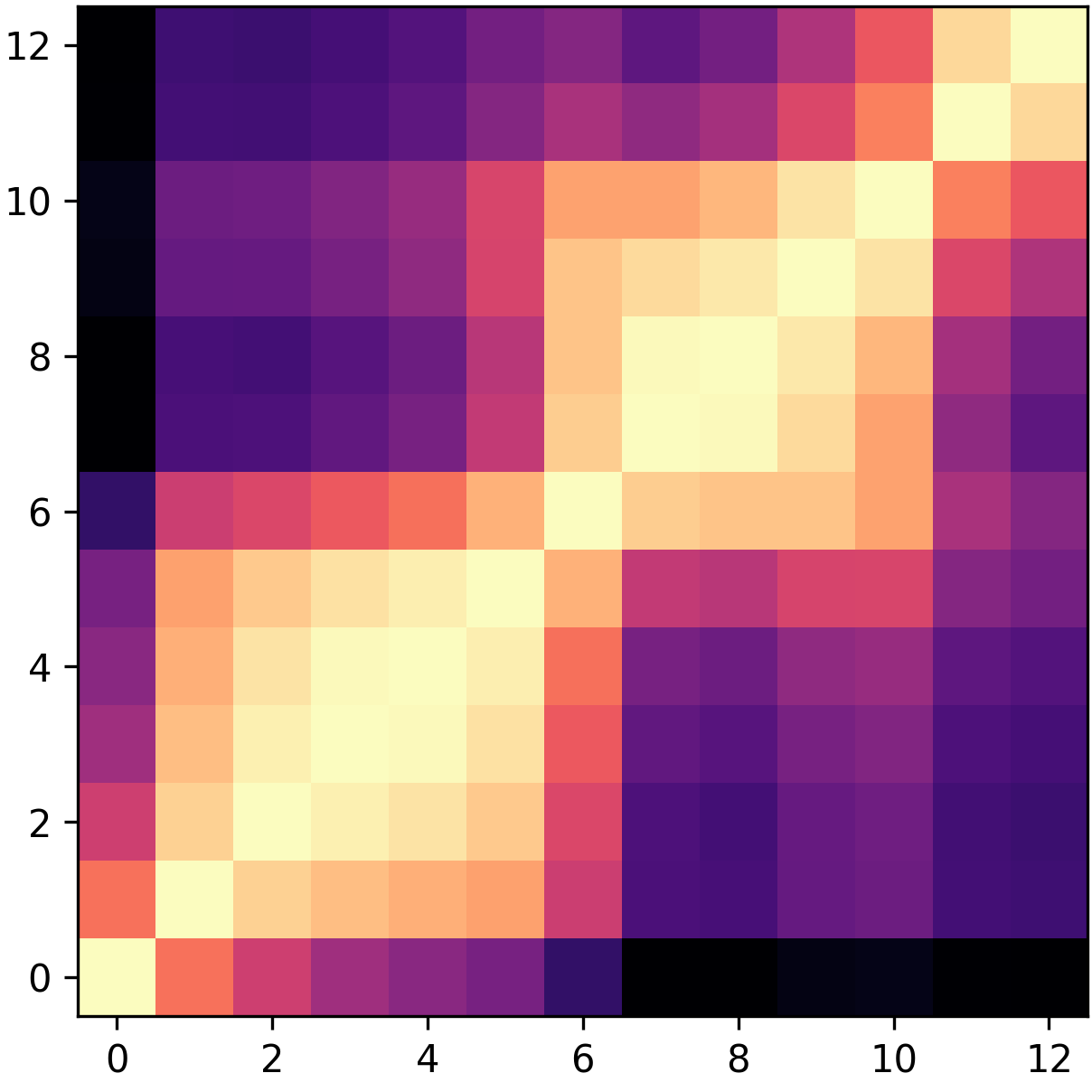}
			\label{fig:base}
		\end{minipage}%
	}%
	\subfigure[w/ CKA-SR]{
		\begin{minipage}[t]{0.3\linewidth}
			\centering
			\includegraphics[width=\linewidth]{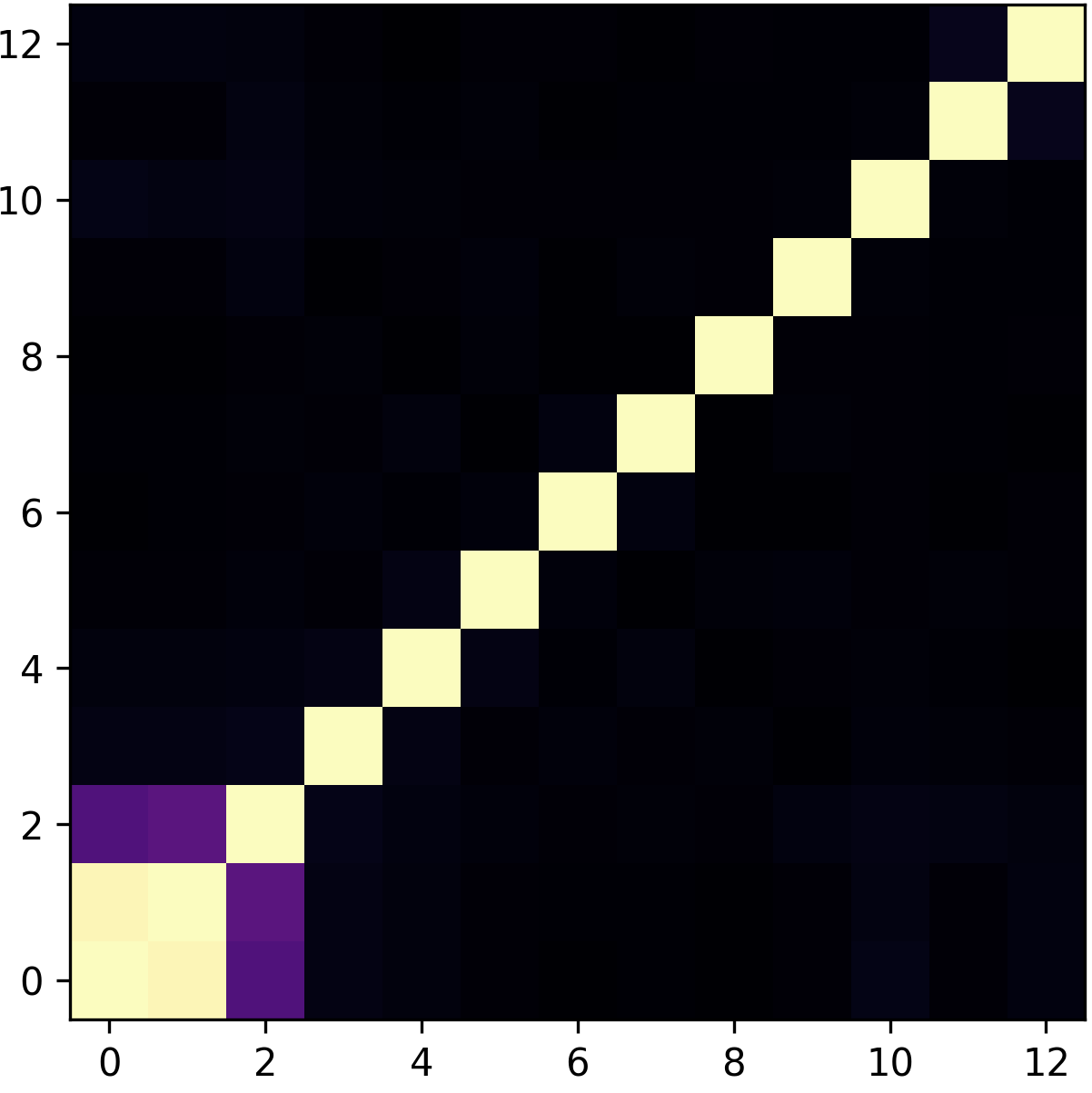}
			\label{fig:cka}
		\end{minipage}%
	}%
	\subfigure[weight distribution]{
		\begin{minipage}[t]{0.3\linewidth}
			\centering
			\includegraphics[width=\linewidth]{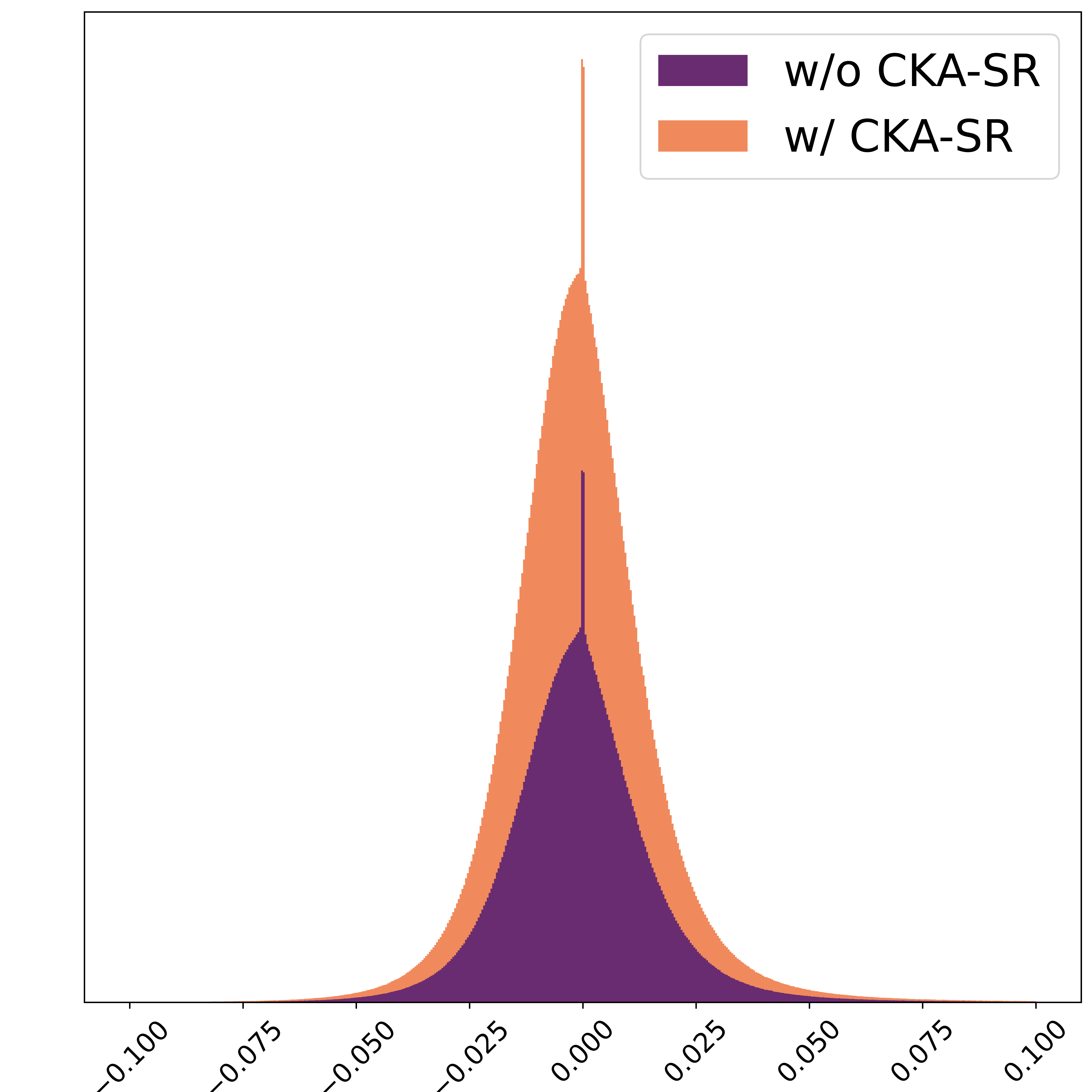}
			\label{fig:distribution}
		\end{minipage}%
	}%
	\subfigure{
		\begin{minipage}[t]{0.06\linewidth}
			\centering
			\includegraphics[width=\linewidth]{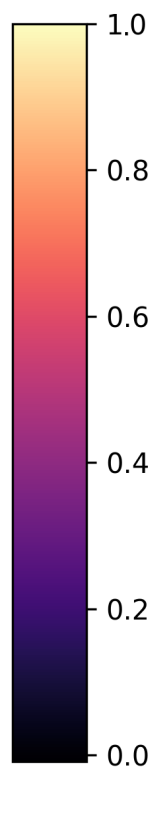}
		\end{minipage}%
	}%
	\caption{Reduction of interlayer feature similarity with CKA-SR. (a) Interlayer feature similarity visualization of baseline models. (b) Interlayer feature similarity visualization of models pre-trained with CKA-SR. (c) Comparison of weight distribution between baseline and CKA-SR models.}
	\label{fig:vis}
	\vspace{-10pt}
\end{figure*}

To validate the proposed CKA-SR, we conduct experiments on several advanced sparse training methods, such as Lottery Ticket Hypothesis (LTH) \cite{frankle2018lottery}, Gradient Signal Preservation (GraSP) \cite{wang2020picking}, Dual Lottery Ticket Hypothesis (DLTH) \cite{bai2022dual}, and Random Sparse Training \cite{liu2021unreasonable}. Specifically, we introduce our CKA-SR regularization to the training process of these sparse training methods and thus achieve consistent performance gains across these methods. Moreover, we introduce CKA-SR to the training and finetuning process of network pruning methods such as l1-norm filter pruning \cite{li2016pruning}, non-structured weight-level pruning~\cite{han2015learning}, and knapsack channel pruning \cite{aflalo2020knapsack}, and thus achieve performance improvements. In short, CKA-SR boosts the performance of sparse training and network pruning methods. Appendix and codes are included in the supplementary materials. See them in \href{https://anonymous.4open.science/r/Learning-Sparse-Neural-Networks-with-Identity-Layers-9369}{https://anonymous.4open.science/r/Learning-Sparse-Neural-Networks-with-Identity-Layers-9369}.

Our contributions are four-fold:

\begin{itemize}
    \item We heuristically investigate the intrinsic link between interlayer feature similarity and network sparsity. To the best of our knowledge, we are the first to find that reducing interlayer feature similarity directly increases network sparsity.
    \item Theoretically, we prove the equivalence of interlayer feature similarity reduction, interlayer mutual information reduction, and network sparsity increment.
    \item We proposed Identity Layers Regularization (ILR) with few-shot samples increases network sparsity and weakens overparameterization by explicitly reducing interlayer feature similarity. Specifically, we implement ILR as CKA-SR.
    \item Experimentally, our CKA-SR regularization term increases network sparsity and improves the performance of multiple sparse training methods and several pruning methods.
\end{itemize}

%% file: Tex/03_Preliminaries.tex
\section{Related Works and Preliminaries}
\label{sec: pre}

\subsection{Centered Kernel Alignment}
\label{sec: CKA}
Here we provide the formalization of Centered Kernel Alignment (CKA). 
For the feature map $X\in\mathbb{R}^{n\times p_1}$ and feature map $Y\in\mathbb{R}^{n\times p_2}$ (where $n$ is the number of examples, while $p_1$ and $p_2$ are the number of neurons), we use kernels $k$ and $l$ to transform $X$ and $Y$ into $K$ and $L$ matrices, where the elements are defined as: $K_{ij} = k(x_i, x_j), L_{ij} = l(y_i, y_j)$.
Further, the formalization of CKA-based similarity measurement $\mathcal{F}$ of $K$ and $L$ matrices could be formulated as:
\begin{align}
\mathbf{CKA}(K,L) = \frac{\mathrm{HSIC}(K,L)}{\sqrt{\mathrm{HSIC}(K,K)\mathrm{HSIC}(L,L)}}
\end{align}
where $\mathrm{HSIC}$ is the empirical estimator of Hilbert-Schmidt Independence Criterion \cite{gretton2005measuring}. Then, the formalizations of CKA-based similarity measurement for linear kernel $k(x, y) = x^Ty$ 
is as follows: 
\begin{align}
\mathbf{CKA}_{Linear}(X,Y) = \frac{||Y^TX||_F^2}{||X^TX||_F||Y^TY||_F}
\end{align}



\subsection{Interlayer feature similarity of overparameterized models}
\label{sec: interlayer}

Nguyen \textit{et al.}~\cite{nguyen2020wide} investigate the relationship between overparameterized models and similar feature representations. Specifically, wide ResNets, deep ResNets and ResNets trained on small datasets possess extremely similar feature representations between adjacent layers, named block structure. 
Then they infer an empirically verified hypothesis that \textit{overparameterized models possess similar feature representations}. Besides, similar observations also appear in ViT \cite{dosovitskiy2020image} based architectures. 
We may conclude that such block structure is a common problem in different architectures. This prompts us to explore the potential benefits of reducing interlayer feature similarity and learning sparse neural networks with identity layers.

%% file: Tex/04_Methodology.tex
\section{Methodology}
\label{sec: meth}


\subsection{Sparsity regularization based on Centered Kernel Alignment}


As discussed above, the interlayer feature similarity of overparameterized models motivates us to learn sparse neural networks with identity layers. We choose Centered Kernel Alignment (CKA) as the basis of our method, for it's widely applied to measuring feature similarity of different layers. On the other side, the high similarity of layers indicates the overparameterization of Deep Neural Networks. Hence, CKA similarity measurement could be regarded as a scale of overparameterization. This reminds us of directly reducing this measurement to solve overparameterization problems. Even more remarkable, CKA owns many excellent properties, including robustness, invariance to orthogonal transformation, and invariance to scale transformation. These properties make CKA ideal for designing a regularization term to solve overparameterization problems.

Specifically, we add a CKA-based regularization term to the training loss function. For a model with empirical loss (cross-entropy loss) $\mathcal{L_{\mathcal{E}}}$, the training loss with CKA-SR is formalized as:
\begin{align}
    \label{eqn:cka-sr}
    \mathcal{L} = \mathcal{L_{\mathcal{E}}} + \mathcal{L_{\mathcal{C}}} = \mathcal{L_{\mathcal{E}}} + \beta \cdot \sum_{s=1}^{S} \sum_{i=0}^{N_s} \sum_{j=0, j\neq i}^{N_s} w_{ij} \mathbf{CKA}_{Linear}(X_i,X_j)
\end{align}
where $\mathcal{L_{\mathcal{C}}}$ is CKA-SR and $\beta$ is the weight of $\mathcal{L_{\mathcal{C}}}$. $S$ is the number of stages in the network. For networks with only one stage such as DeiTs, $N_s$ is the total number of layers. And for networks with several stages such as ResNets, $N_s$ is the number of layers in each stage $s$. $w_{ij}$ is the weight of CKA measurement between the $i^{th}$ and the $j^{th}$ layer, and it's optional. $X_0$ is the input representation and $X_i$ is the output representation of the $i^{th}$ layer.



The $\mathcal{L_{\mathcal{C}}}$ part in Eq.\eqref{eqn:cka-sr} forcibly reduces the sum of the pairwise similarity of all layers in the network, \emph{i.e.} forcibly reduces the interlayer similarity of the network.


\subsection{Theoretical analysis}
\label{sec: Theoretical Analysis} 

\subsubsection{Approximate sparsity.}
To further explore the relationship between the Frobenius norm of weight matrix and network sparsity, we expand sparsity to approximate sparsity. We define $\epsilon$-sparsity (\emph{i.e.}, approximate sparsity) of a neural network as follows: \\
\begin{align}
\label{Eqn:eS}
    S_{\epsilon} = \frac{|\{w|w \in \mathbb{W} \land |w|<\epsilon\}|}{|\mathbb{W}|}
\end{align}
where $\epsilon$ is a number close to zero, $\mathbb{W}$ is the set consisting of all parameters of the network's weight matrix, $|\mathbb{W}|$ is the total number of parameters, and $\{w|w \in \mathbb{W} \land |w|<\epsilon\}$ is the set consisting of small parameters (\emph{i.e.}, parameters with an absolute value smaller then $\epsilon$) of the weight matrix. 

In Eq.~\eqref{Eqn:eS}, $S_{\epsilon}$ represents the proportion of network parameters that approach 0. We define this as $\epsilon$-sparsity of the network. Further, we prove that $\epsilon$-sparsity and sparsity (\emph{i.e.}, proportion of network parameters that equal 0) of neural networks are approximately equivalent in practice. Our theory is formulated as Theorem \ref{Thm: theorem1}. See the detailed proof of Theorem \ref{Thm: theorem1} in the Appendix.

\begin{theorem}
\label{Thm: theorem1}
The $\epsilon$-sparsity and the sparsity of neural networks are approximately equivalent.
\end{theorem}

\subsubsection{Information bottleneck.}
The information bottleneck (IB) theory proposed by Tishby \textit{et al.}~\cite{tishby2000information} is an extension of the rate distortion theory of source compression. 
This theory shows a trade-off between preserving relevant label information and obtaining efficient compression. Tishby \textit{et al.}~\cite{tishby2015deep} further research the relationship between information bottleneck theory and deep learning. They interpret the goal of deep learning as an information-theoretic trade-off between compression and prediction. 
According to the principles of information bottleneck theory, for a neural network $Y = f(X)$ with input $X$ and output $Y$, the best representation of intermediate feature map $\hat{X}$ captures the relevant features and ignores the irrelevant features (features that have little contribution to the prediction of $Y$) at the same time. This process is called "compression". One of its minimization objectives is as follows:
\begin{align}
\label{Eqn:min}
    L = I(X;\hat{X}) - \alpha I(\hat{X};Y)
\end{align}
where $I(X;\hat{X})$ is the mutual information between input $X$ and intermediate representation $\hat{X}$, $I(\hat{X};Y)$ is the mutual information between intermediate representation $\hat{X}$ and output $Y$, and $\alpha$ is a weight parameter for adjusting their proportions.

\subsubsection{Minimizing the mutual information.}
Firstly, we prove that our CKA-SR is continuous and optimizable in Theorem \ref{Thm: theorem2}, which makes it possible to minimize CKA-SR in machine learning. See the detailed proof of Theorem \ref{Thm: theorem2} in the Appendix. Then we prove that minimizing CKA-SR minimizes the mutual information $R = I(X;\hat{X})$ between the intermediate and input representation. Besides, the $\alpha I(\hat{X};Y)$ part of Eq.~\eqref{Eqn:min} is implicitly optimized through the cross entropy loss $\mathcal{L_{\mathcal{E}}}$. Thus, we prove that our method minimizes the optimization objective in Eq.~\eqref{Eqn:min}, \emph{i.e.}, our CKA-SR method conforms to the principles of information bottleneck theory, and it's beneficial to the representation compression process. Our theory is formulated as Theorem \ref{Thm: theorem3}. 

\begin{theorem}
\label{Thm: theorem2}
$\mathcal{L_{\mathcal{C}}}$ is continuous and optimizable.
\end{theorem}

\begin{theorem}
\label{Thm: theorem3}
Minimizing $\mathcal{L_{\mathcal{C}}}$ minimizes the mutual information $R = I(X;\hat{X})$ between intermediate representation $\hat{X}$ and input representation $X$.
\end{theorem}

To prove Theorem \ref{Thm: theorem3}, we first review Lemma \ref{Thm: lemma1} and Lemma \ref{Thm: lemma2} from \cite{zheng2021information} as follows. Following ~\cite{zheng2021information}, we assume that $X \sim \mathcal N(\textbf{0}, \boldsymbol \Sigma_X)$ and $Y \sim \mathcal N(\textbf{0}, \boldsymbol \Sigma_Y)$, \textit{i.e.}, feature maps $X$ and $Y$ follow Gaussian distribution.

\begin{lemma}
\label{Thm: lemma1}
Minimizing the distance between $X^TY$ and zero matrix is equivalent to minimizing the mutual information $I(X; Y)$ between representation $X$ and $Y$. 
\end{lemma}
\begin{lemma}
\label{Thm: lemma2}
Minimizing $\mathbf{CKA}_{Linear}(X, Y)$ is equivalent to minimizing $I(X; Y)$.
\end{lemma}
These two lemmas illustrate the relationship between the CKA similarity measurement and information theory. That is, \textit{minimizing the CKA similarity between two feature representations is equivalent to minimizing the mutual information between them}.
Based on these two lemmas, we prove Theorem \ref{Thm: theorem3}. See the detailed proof of the two lemmas and Theorem \ref{Thm: theorem3} in the Appendix.

Theorem \ref{Thm: theorem3} connects CKA-SR with information bottleneck theory. 
In short, \textit{minimizing CKA-SR is equivalent to optimizing the optimization objective \\
$I(X;\hat{X})$ of information bottleneck theory}. 

\subsubsection{Increasing the sparsity of neural networks.}
Further, starting from the information bottleneck theory, we prove that CKA-SR increases the network sparsity, formulated as Theorem \ref{Thm: theorem4}.

\begin{theorem}
\label{Thm: theorem4}
Minimizing $R = I(X; \hat{X}) \Leftrightarrow $ Minimizing
$||W||_F^2 \Leftrightarrow$ Increasing the approximate sparsity of network $\Leftrightarrow$ Increasing network sparsity.
\end{theorem}
\begin{proof}
According to Theorem~\ref{Thm: theorem3}, CKA-SR minimizes $R = I(X;\hat{X})$ for any $X$. Further, combining this with Lemma~\ref{Thm: lemma1}, for any $X$, CKA-SR minimizes the distance between $X^T\hat{X}$ and 0 matrix. For a fully-connected layer, we have $\hat{X} = W^TX+b$. Hence, due to the discussions above, we have: for any $X$, CKA-SR minimizes the distance between $X^T(W^TX+b) = X^TW^TX+X^Tb$ and 0 matrix. We take an orthogonalized $X$. Due to the unitary invariance (\emph{i.e.}, orthogonal invariance in the real number field) of Frobenius norm, $||W||_F^2$ equals to $||X^TW^TX||_F^2$. Therefore, minimizing the distance between $X^TW^TX+X^Tb$ and 0 matrix is equivalent to minimizing $||X^TW^TX||_F^2$ and further equivalent to minimizing $||W||_F^2$. 

The above minimization of $||W||_F^2$ minimizes the norm of parameter values in weight matrix $W$, thus making the values more concentrated around 0 value. This increases the network's approximate sparsity (defined earlier in this article). Further, according to Theorem~\ref{Thm: theorem1}, the approximate sparsity and sparsity are approximately equivalent. So we prove that the above minimization of $||W||_F^2$ increases the network sparsity.
\end{proof}
Theorem~\ref{Thm: theorem4} connects the optimization objective of information bottleneck theory with network sparsity, thus connecting CKA-SR with network sparsity. 
In short, \textit{CKA-SR models are more sparse.} We validate this conclusion with our experimental results. Fig.\ref{fig:distribution} compares parameter distribution between CKA-SR and baseline models. It's evident that the absolute value of CKA-SR network parameters is more concentrated around 0.

%% file: Tex/05_Experiments.tex
\section{Experiments}
\label{sec: Experiments}

\subsection{Implementations}
\subsubsection{Datasets and backbone models.}
We validate the effectiveness of our CKA-SR method on image classification, network pruning, and advanced sparse training. We use ResNet18, ResNet20, ResNet32 and ResNet50~\cite{he2016deep} as backbones to conduct extensive experiments on CIFAR-10, CIFAR-100 and ImageNet datasets.

\subsubsection{Implementations.}
We implement our CKA-SR as a regularization of the loss function. We develop a plug-and-play CKA-SR class in PyTorch and plug it into various pre-training and sparse training codes. 
Because CKA-SR is a regularization of layerwise parameters instead of feature maps themselves, we could utilize few-shot samples of each batch (\textit{generally 8 samples when the batch size is 128 or 256}) to compute CKA-SR. This reduces the computational complexity, thus reducing training expenses.
Precisely, we strictly follow the experimental settings of the pruning~\cite{aflalo2020knapsack, han2015learning, li2016pruning} and sparse training methods~\cite{bai2022dual, frankle2018lottery, liu2021unreasonable, wang2020picking} and make fair comparisons with them using CKA-SR. The total number of epochs, batch size, optimizer, weight decay, and learning rates all stay the same with the methods to be compared with. 


\subsection{Pre-Training with CKA-SR}
As previously proved, our CKA-SR increases network sparsity. So we validate the performance of CKA-SR in network pruning tasks. We directly prune models pre-trained with CKA-SR on large-scale datasets such as ImageNet. We carry out experiments on several pruning methods and find that our method is effective. As shown in Figure~\ref{fig:performances}, at the same pruning ratio, CKA-SR models outperform baseline models. 

\begin{figure*}[!tb]
	\centering
	\subfigure[L1-norm filter pruning]{
		\begin{minipage}[t]{0.32\linewidth}
			\centering
			\includegraphics[width=\linewidth]{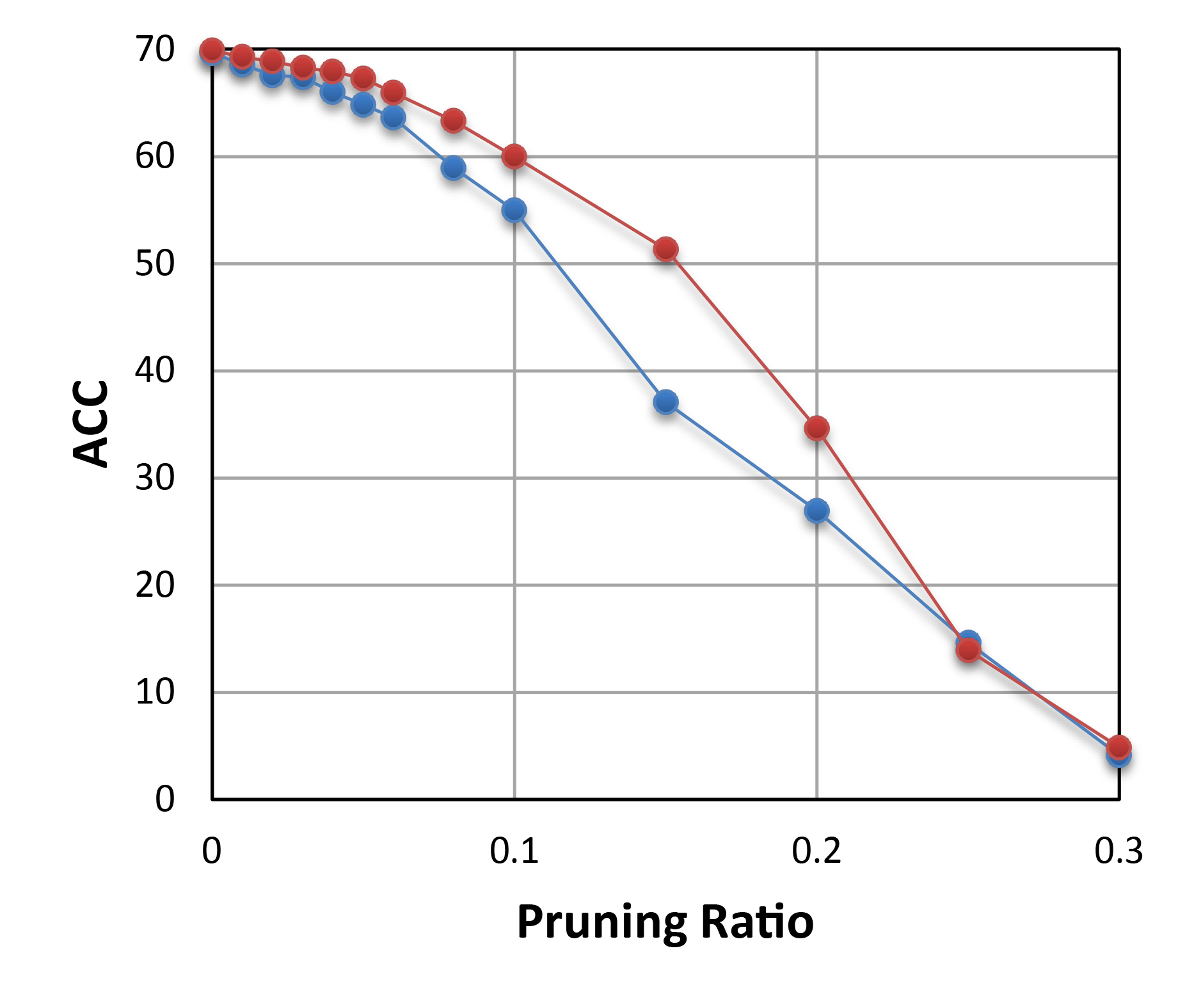}
			\label{fig:l1}
		\end{minipage}%
	}%
	\subfigure[Knapsack pruning]{
		\begin{minipage}[t]{0.32\linewidth}
			\centering
			\includegraphics[width=\linewidth]{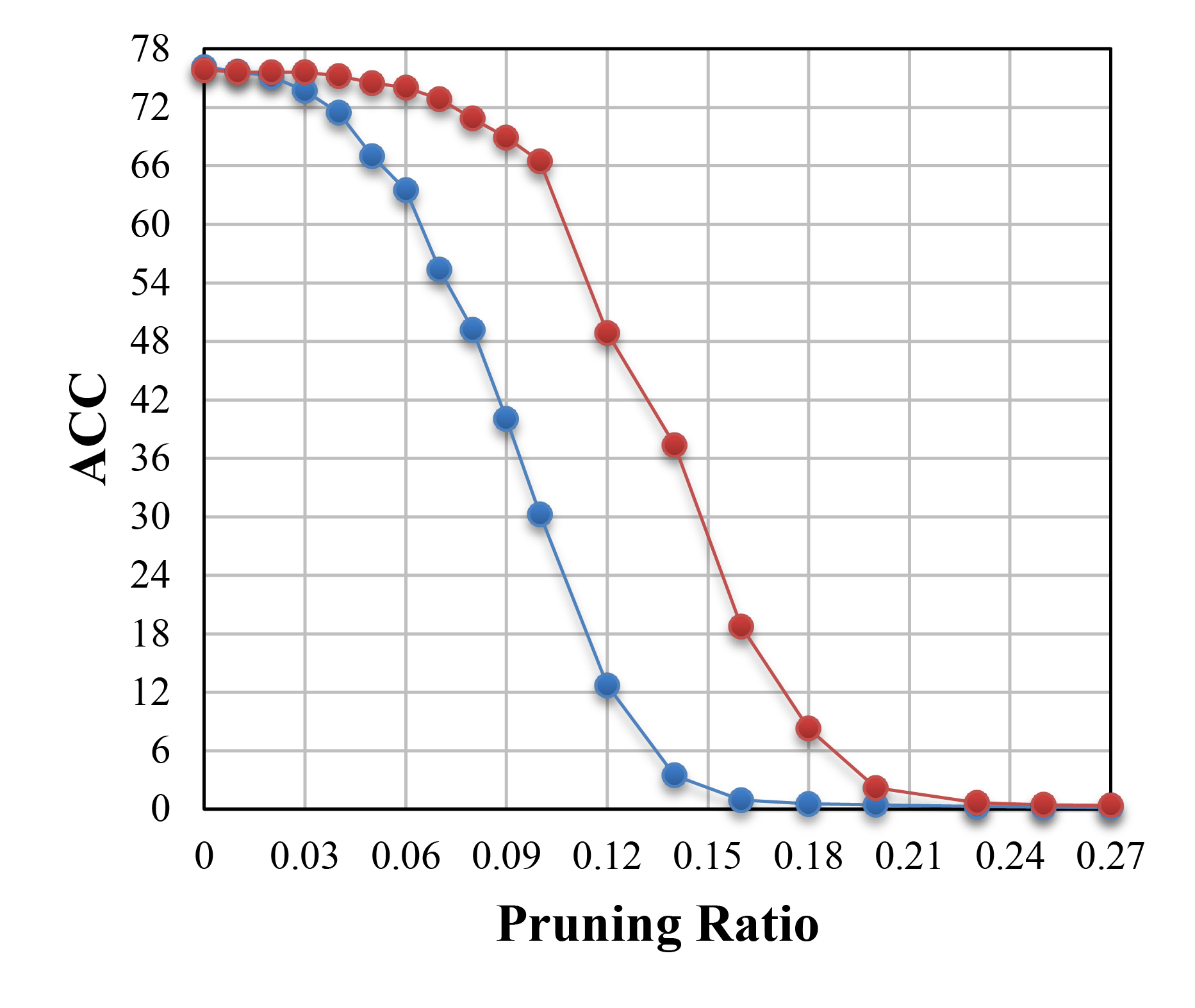}
			\label{fig:knapsack}
		\end{minipage}%
	}%
	\subfigure[Weight-level pruning]{
		\begin{minipage}[t]{0.32\linewidth}
			\centering
			\includegraphics[width=\linewidth]{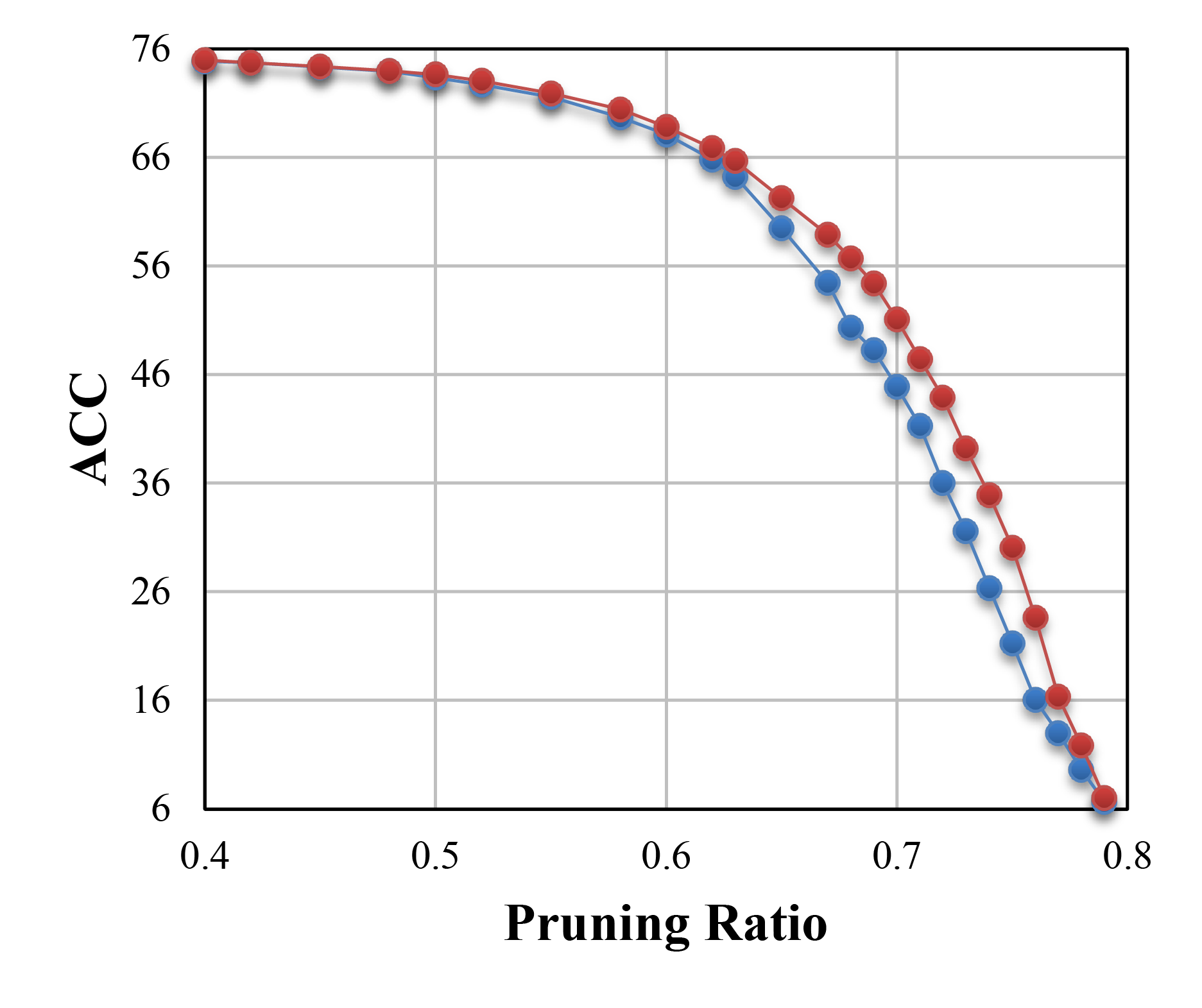}
			\label{fig:weight}
		\end{minipage}%
	}%
	\caption{Performances of several pruning methods with CKA-SR. The red lines represent CKA-SR models and the blue lines represent baseline models. (a) Performances of L1-norm filter pruning with ResNet18 on ImageNet.  (b) Performances of knapsack channel pruning with ResNet50 on ImageNet. (c) Performances of non-structured weight-level pruning with ResNet50 on ImageNet.}
	\label{fig:performances}
	\vspace{-10pt}
\end{figure*}

\subsubsection{Structured pruning.}
Following the setting of ~\cite{li2016pruning}, we perform filter pruning on models pre-trained with CKA-SR without finetuning. Specifically, we prune the filter according to the L1-Norm. The relationship between the pruning ratio and performance is shown in Figure \ref{fig:l1}. When a few filters are pruned, the performance reduction of CKA-SR models is significantly smaller than that of baseline models. 


As a State-Of-The-Art method for channel pruning, we perform Knapsack channel pruning~\cite{aflalo2020knapsack} on models pre-trained with CKA-SR and achieve higher classification accuracy. The results of Knapsack pruning (w/o finetuning) are shown in Figure \ref{fig:knapsack}. When a few channels are pruned, the performance reduction of CKA-SR models is much smaller than that of baseline models, which means CKA-SR models possess much higher sparsity.


\subsubsection{Non-structured pruning.}
We perform non-structured weight-level pruning \cite{han2015learning} according to the absolute values of individual weights and compare the performance between baseline ResNet models and pre-trained ResNets with CKA-SR. The relationship between pruning ratio and performance is shown in Figure \ref{fig:weight}. It could be concluded that when massive weights are pruned, the performance reduction of CKA-SR models is smaller than that of baseline models.


Generally, pre-trained models with CKA-SR outperform baseline models in both structured and non-structured pruning methods.

\subsection{Sparse network training with CKA-SR}
We conduct extensive experiments on several State-Of-The-Art sparse training methods. 
For fair comparisons, our experiments follow the same settings and backbones of these methods~\cite{bai2022dual, frankle2018lottery, liu2021unreasonable, wang2020picking}. 
Note that we conduct experiments on extremely high sparsity (such as 99.8\%) settings in GraSP~\cite{wang2020picking}, Random sparse training~\cite{liu2021unreasonable}, and DLTH~\cite{bai2022dual}. From Table \ref{tab:sparse}, we can find that CKA-SR consistently improves the performance at different levels of sparsity ratios in LTH~\cite{frankle2018lottery}, GraSP~\cite{wang2020picking}, Random sparse training~\cite{liu2021unreasonable}, and DLTH~\cite{bai2022dual}.

\subsubsection{LTH.}
Lottery Ticket Hypothesis (LTH) \cite{frankle2018lottery} is proposed to train a sparse network from scratch, which states that any randomly initialized dense network contains sub-networks achieving similar accuracy to the original network. We plug our CKA-SR into the training process of LTH. We use the code implemented for LTH by~\cite{su2020sanity}, adopt ResNet32 as the backbone, and apply sparsity ratios from 0.70 to 0.98 for fair comparisons. The results are given in the first row of Table~\ref{tab:sparse}. 


\subsubsection{GraSP.}
Gradient Signal Preservation (GraSP) \cite{wang2020picking} proposes to preserve the gradient flow through the network during sparse training. We plug our CKA-SR into the sparse training process of GraSP, adopt ResNet32 as the backbone, and apply sparsity ratios from 0.70 to 0.998. The results are given in the second row of Table~\ref{tab:sparse}.


\subsubsection{Random sparse training.} 
As one of the newest and State-Of-The-Art sparse training methods, it has been proven that sparse training of randomly initialized networks can also achieve remarkable performances~\cite{liu2021unreasonable}. We plug our CKA-SR into the random sparse training process, adopt ResNet20 as the backbone, and apply sparsity ratios from 0.70 to 0.998. The results are given in the third row of Table~\ref{tab:sparse}.

\begin{table}[!tb]
  \small
  \caption{The accuracy (\%) when plugging CKA-SR to different sparse training methods on CIFAR-100 from scratch. (LTH is broken when sparsity ratio is larger than 0.99 due to destruction of the structure.)}
  \label{tab:sparse}
  \centering
  \begin{adjustbox}{max width=\linewidth}
  \begin{tabular}{|l|l|l|l|l|l|l|l|}
    \hline
     \cellcolor{white}\multirow{2}{*}{Backbone} & \cellcolor{white}\multirow{2}{*}{Method} & \multicolumn{6}{c|}{Sparsity} \\
    \cline{3-8}
       &  & \textit{0.70} & \textit{0.85} & \textit{0.90} & \textit{0.95} & \textit{0.98} & \textit{0.998}\\
     \hline
    \multirow{2}{*}{ResNet32} & LTH\cite{frankle2018lottery} & 72.28 & 70.64 & 69.63 & 66.48 & 60.22 & \ding{56} \\
     & +CKA-SR & \textbf{72.67} & \textbf{71.90} & \textbf{70.11} & \textbf{67.07} & \textbf{60.36} & \ding{56}    \\
     \hline
     \multirow{2}{*}{ResNet32} & GraSP\cite{wang2020picking} & 71.98 & 70.22 & 69.19 & 65.82 & 59.46 &12.19 \\
     & +CKA-SR & \textbf{72.19} & \textbf{70.25} & \textbf{69.28} & \textbf{66.29} & \textbf{59.49} & \textbf{18.44}    \\
    \hline
    \multirow{2}{*}{ResNet20} & Random\cite{liu2021unreasonable} & 65.42 & 60.37 & 56.96 & 47.27 & 33.74 & 2.95 \\
     & +CKA-SR & \textbf{65.60} & \textbf{60.86} & \textbf{57.25} & \textbf{48.26} & \textbf{34.44} &\textbf{3.32}  \\
    \hline
    \multirow{2}{*}{ResNet20} & DLTH\cite{bai2022dual} & 67.63 & 65.33 & 62.90 & 57.33 & 48.08 & 19.32\\
     & +CKA-SR & \textbf{67.95} & \textbf{65.80} & \textbf{63.19} & \textbf{57.99} & \textbf{49.26} & \textbf{20.81}\\
    \hline
  \end{tabular}
  \end{adjustbox}
  \vspace{-10pt}
\end{table}


\subsubsection{DLTH.}
 As one of the newest and State-Of-The-Art LTH-based sparse training methods, Dual Lottery Ticket Hypothesis (DLTH)~\cite{bai2022dual} proposes to randomly select subnetworks from a randomly initialized dense network, which can be transformed into a trainable condition and achieve good performance. We apply our CKA-SR to the training process of the DLTH method, adopt ResNet20 as the backbone, and apply sparsity ratios from 0.70 to 0.998. The results are given in the final row of Table \ref{tab:sparse}.


As shown in Table \ref{tab:sparse}, our CKA-SR can be plugged into multiple sparse training methods and improves the model performance consistently. 
The CKA-SR is effective consistently at different sparse networks, especially at extremely high sparsity. 
For GraSP, CKA-SR achieves more than 4.0\% of performance improvement at sparsity 99.5\% and 6.0\% at sparsity 99.8\%.

\subsection{Ablation studies}
\subsubsection{Ablation study of regularization term.} 
Savarese \textit{et al.}~\cite{savarese2020winning} develop a \\ regularization-based sparse network searching method named Continuous Sparsification. This method introduces $L_0$ Regularization into sparse training. We compare our CKA-SR with $L_0$ Regularization theoretically and experimentally. Theoretically, CKA-SR and $L_0$ regularization regularize networks from different granularity levels. $L_0$ regularization regularizes networks from the individual parameter level, while CKA-SR regularizes networks from the layer level. These regularizations from different granularity levels could work together. Experimentally, we conduct sparse training experiments with ResNet18 on CIFAR-10 using the official code of the CS method. We find that our CKA-SR is able to replace $L_0$ regularization and achieves better performance. Besides, combining CKA-SR and $L_0$ improves performance by 0.4\%, demonstrating that our CKA-SR could cooperate with other regularizations. The results are shown in Table \ref{sample-table10}. \\

\begin{table}[H]
 \vspace{-10pt}
  \small
  \caption{Ablation study of regularization terms}
  \label{sample-table10}
  \centering
  \begin{adjustbox}{max width=\linewidth}
  \begin{tabular}{|l |c |c |c |}
    \hline
     Settings & CKA-SR Only & $L_0$ Only & CKA-SR+$L_0$\\
    \hline
     Top1-Acc & 91.63 & 91.56 & 91.92\\
    \hline
  \end{tabular}
  \end{adjustbox}
  \vspace{-10pt}
\end{table}

\subsubsection{Ablation study of hyperparameter $\beta$.} 
We conduct the ablation study of hyperparameter $\beta$ with Random Sparse Training~\cite{liu2021unreasonable} method on CIFAR-10 dataset. Taking ResNet20 model at a sparsity of 0.95 and adjusting the weight hyperparameter $\beta$ of our CKA-SR, we get the results shown in Table \ref{sample-table9}.  \\

\begin{table}[H]
 \vspace{-10pt}
  \small
  \caption{Ablation study of hyperparameter $\beta$}
  \label{sample-table9}
  \centering
  \begin{adjustbox}{max width=\linewidth}
  \begin{tabular}{|l |l |l |l |l |l |l |l |l |l |}
    \hline
     $\beta$ & 0 & 1e-05 & 5e-05 & 2e-04 & 8e-04 & 1e-03 & 2e-03 & 5e-03\\
    \hline
     Top1-Acc & 84.16 & 84.69 & 84.42 & 84.40 & \textbf{85.03} & 84.82 & 84.08 & 83.86\\
    \hline
  \end{tabular}
  \end{adjustbox}
  \vspace{-10pt}
\end{table}

We conclude that multiple values of hyperparameter $\beta$ between 1e-05 and 1e-03 increase the performance of sparse networks. However, when the hyperparameter $\beta$ becomes too large, it would weaken the succession of information through layers, thus causing a reduction in performance. That is to say, there is a trade-off between the identity of layers and the succession of information through layers. In the view of sparsity, there is a trade-off between high sparsity and ideal performance.

%% file: Tex/06_Conclusion.tex
\section{Conclusion}
\label{sec: conc}

Our work reveals the relationship between overparameterization, network sparsity, and interlayer feature similarity. We thus propose to use the robust and advanced CKA similarity measurement to solve the overparameterization issue. Specifically, we propose a plug-and-play sparsity regularization named CKA-SR which explicitly reduces interlayer similarity. Theoretically, we reveal the equivalence of reducing interlayer similarity and increasing network sparsity, thus proving the CKA-SR increases network sparsity. Experimentally, our CKA-SR consistently improves the performances of several State-Of-The-Art sparse training methods and several pruning methods. Besides, our CKA-SR outperforms former regularization methods. 
In the future, considering our limitations of expenses to manually select hyperparameters and calculate loss, we will continue to investigate the cooperation of multiple regularizations in sparse training and reduce the expenses of sparse training. 



%% file: Tex/02_Related_Works.tex
\section{Related works}
\label{sec: related}

\subsection{Network sparsity and sparse training}

Overparameterization is a common problem in Deep Neural Networks, meaning many redundant parameters exist in models. In machine learning, overparameterization usually means overfitting and a colossal waste of resources. Therefore, we need model sparsification methods to explore the redundancy of model parameters and compress the model by reducing these redundancy parameters. 
Network pruning methods, including non-structured and structured pruning, are the most common ways to compress models.

Further, on the theoretical basis of pruning and sparsifying a pre-trained network, Frankle and Carbin \cite{frankle2018lottery} propose the Lottery Ticket Hypothesis (LTH), which enables to train a sparse network from scratch. The Lottery Ticket Hypothesis states that any randomly initialized dense network contains sub-networks with the following properties. When independently trained, the initialized sub-network can achieve similar accuracy to the original network. 
On the basis of LTH, various sparse training methods have been proposed.


\subsection{Representation similarity measurements}
The representational similarity is an important concept in machine learning. General methods (including Euclidean distance, Manhattan distance, Cosine distance, etc.) compute statistics of two feature representations or utilize the geometric properties to calculate representation similarity.

However, when we need to measure the feature similarity between layers of neural networks, the above measurements are not robust and are sensitive to disturbance due to the randomness of neural networks. Frontier researchers propose two categories of methods to measure feature similarity between layers, including Canonical Correlation Analysis (CCA)~\cite{hardoon2004canonical, hotelling1992relations, ramsay1984matrix} and Centered Kernel Alignment (CKA)~\cite{kornblith2019similarity}. 
Briefly, CCA method maps the two feature matrices to be measured onto a set of bases and then maximizes the correlation. 

\subsection{Information bottleneck theory}

The information bottleneck (IB) theory proposed by Tishby \textit{et al.}~\cite{tishby2000information} is an extension of the rate distortion theory of source compression. 
This theory shows a trade-off between preserving relevant label information and obtaining efficient compression. Tishby \textit{et al.}~\cite{tishby2015deep} further research the relationship between information bottleneck theory and deep learning. They interpret the goal of deep learning as an information-theoretic trade-off between compression and prediction. 